\documentclass{article}
\PassOptionsToPackage{numbers, compress}{natbib}

\usepackage[final]{neurips_2019}

\usepackage[utf8]{inputenc}
\usepackage[T1]{fontenc}
\usepackage{hyperref}
\usepackage{url}
\usepackage{booktabs}
\usepackage{amsfonts}
\usepackage{nicefrac}
\usepackage{microtype}

\usepackage{graphicx}
\usepackage{subfig}

\usepackage{url}
\usepackage[dvipsnames]{xcolor} 
\usepackage{amsmath}
\usepackage{algorithmic}
\usepackage{algorithm}
\usepackage{amsthm}
\usepackage{cleveref}
\usepackage{array}
\usepackage{verbatim}
\usepackage{natbib}
\usepackage{wrapfig}

\theoremstyle{plain}
\newtheorem{prop}{\protect\propositionname}
\theoremstyle{plain}
\newtheorem{thm}{\protect\theoremname}
\theoremstyle{plain}
\newtheorem{assumption}{\protect\assumptionname}
\theoremstyle{plain}
\newtheorem{lem}[thm]{\protect\lemmaname}
\theoremstyle{plain}
\newtheorem{cor}[thm]{\protect\corollaryname}
\providecommand{\assumptionname}{Assumption}
\providecommand{\lemmaname}{Lemma}
\providecommand{\propositionname}{Proposition}
\providecommand{\corollaryname}{Corollary}
\providecommand{\theoremname}{Theorem}

\newtheorem{example}{Example}

\clearpage{}%

\newcommand{\argmax}{\mathop{\rm argmax}\limits}
\newcommand{\argmin}{\mathop{\rm argmin}\limits}

\newcommand{\boldtheta}{{\boldsymbol{\theta}}}
\newcommand{\bolddelta}{{\boldsymbol{\delta}}}

\newcommand{\boldeta}{{\boldsymbol{\eta}}}

\newcommand{\boldH}{{\boldsymbol{H}}}
\newcommand{\boldA}{{\boldsymbol{A}}}
\newcommand{\boldB}{{\boldsymbol{B}}}
\newcommand{\boldC}{{\boldsymbol{C}}}

\newcommand{\boldmu}{{\boldsymbol{\mu}}}
\newcommand{\boldJ}{{\boldsymbol{J}}}

\newcommand{\boldf}{{\boldsymbol{f}}}
\newcommand{\boldepsilon}{{\boldsymbol{\epsilon}}}

\newcommand{\bolds}{{\boldsymbol{s}}}
\newcommand{\boldone}{{\boldsymbol{1}}}

\newcommand{\boldx}{{\boldsymbol{x}}}

\newcommand{\boldV}{{\boldsymbol{V}}}

\newcommand{\boldphi}{{\boldsymbol{\phi}}}
\newcommand{\boldzero}{{\boldsymbol{0}}}

\newcommand{\boldt}{{\boldsymbol{t}}}

\newcommand{\boldpsi}{{\boldsymbol{\psi}}}

\newcommand{\boldI}{{\boldsymbol{I}}}

\newcommand{\mathbbE}{\mathbb{E}}

\newcommand{\iid}{\stackrel{\mathrm{i.i.d.}}{\sim}}
\newcommand{\dx}{\mathrm{d}\boldx}

\newcommand{\vertiii}[1]{{\left\vert\kern-0.25ex\left\vert\kern-0.25ex\left\vert #1 
    \right\vert\kern-0.25ex\right\vert\kern-0.25ex\right\vert}}\clearpage{}%

\graphicspath{{figs/}}
\newcommand{\ourtitle}{Fisher Efficient Inference of Intractable Models}

\begin{document}

\title{\ourtitle}
\author{
  Song Liu\\
  University of Bristol\\
  The Alan Turing Institute, UK\\
  \texttt{song.liu@bristol.ac.uk}\\
  \And
  Takafumi Kanamori \\
  Tokyo Institute of Technology,\\
  RIKEN, Japan\\
  \texttt{kanamori@c.titech.ac.jp} \\
  \And
  Wittawat Jitkrittum \\
  Max Planck Institute \\
  for Intelligent Systems, Germany\\
  \texttt{wittawat@tuebingen.mpg.de} \\  
  \And
  Yu Chen \\
  University of Bristol, UK\\
  \texttt{yc14600@bristol.ac.uk}\\  
}

\date{}
\maketitle
\begin{abstract}
Maximum Likelihood Estimators (MLE) has many good properties. For example, the asymptotic variance of MLE solution attains equality of the asymptotic Cram{\'e}r-Rao lower bound (efficiency bound), which is the minimum possible variance for an unbiased estimator. 
However, obtaining such MLE solution requires calculating the likelihood function which may not be tractable due to the normalization term of the density model. In this paper, we derive a Discriminative Likelihood Estimator (DLE) from the Kullback-Leibler divergence minimization criterion implemented via density ratio estimation and a Stein operator. We study the problem of model inference using DLE. We prove its consistency and show that the asymptotic variance of its solution can attain the equality of the efficiency bound under mild regularity conditions. 
We also propose a dual formulation of DLE which can be easily optimized. 
Numerical studies validate our asymptotic theorems and we give an example where DLE successfully estimates an intractable model constructed using a pre-trained deep neural network. 
\end{abstract}

\section{Introduction}

Maximum Likelihood Estimation (MLE) has been a classic choice of density parameter estimator.
It can be derived from the Kullback-Leibler (KL)  divergence minimization criterion and the resulting algorithm simply maximizes the likelihood function (log-density function) over a set of observations. The solution of MLE has many attractive asymptotic properties:
the asymptotic variance of MLE solutions reach an asymptotic lower bound of all unbiased estimators \citep{Cramer1946,Rao1945}. 

However, learning via MLE requires evaluating the normalization term of the density
function; it may be challenging to apply MLE to learn a complex model that has a 
computationally intractable normalization term.
A partial solution to this problem is approximating the normalization term or the
gradient of the likelihood function numerically. Many methods along this
line of research have been actively studied: importance-sampling MLE
\citep{Robert2005}, contrastive divergence \citep{Hinton2002} and more recently
amortized MLE \citep{Wang2016}. While the computation of the normalization term is
mitigated, these sampling-based approximate methods come at the expense of
extra computational burden and estimation errors. 

The issue of intractable normalization terms has led to the develoment of other 
approaches other than the KL divergence minimization. For example, Score
Matching (SM) \citep{Hyvaerinen2005} minimizes the Fisher divergence
\citep{Sanchez2012} between the data distribution and a model distribution
which is specified by the gradient (with respect to the input variable) of its
log density function. Its computation does not require the evaluation of the
normalization term, thus SM does not suffer from the intractability issue. 
Extensions of SM has been used for infinite dimensional exponential family models \citep{sriperumbudur2017}, non-negative models \citep{hyvarinen2007,Yu2018} and high dimensional graphical models fitting \citep{lin2016}. 

Other than the Fisher divergence, a kernel-based divergence measure known as Kernel Stein Discrepancy (KSD)
\citep{Chwialkowski2016,LiuQ2016} has been proposed as a test statistic for goodness-of-fit testing to
measure the difference between a data and a model distribution, without the
hassle of evaluating the normalization term. It reformulates the kernel Maximum
Mean Discrepancy (MMD) \citep{gretton2012kernel} with a Stein operator \citep{Stein1972,Gorham2015,oates2017control} which is also
defined using the gradient of the log density function. For the same reason as in SM, the KSD can 
be estimated when applied to a density model with an intractable normalizer. The last few years have
seen many applications of KSD such as variational inference
\citep{LiuQ2016SVGD}, sampling \citep{oates2017control,CheMacGorBriOat2018}, and
score function estimation \citep{LiTur2018, ShiSunZhu2018} among others.
KSD minimization is a natural candidate criterion for fitting intractable
models \citep{Barp2019}. However, the
divergence measure defined by the KSD is directly characterized by the kernel
used. Unlike in the case of goodness-of-fit testing where the kernel may be
chosen by maximizing the test power \citep{Jitkrittum2017}, to date, there is
no clear objective for choosing the right kernel in the case of model fitting.

By contrast, KL divergence has been a classic discrepancy measure for model fitting. 
The question that we address is: can we construct a generic
model inference method by minimizing the KL divergence without the knowledge of
the normalization term?  In this paper, we present a
novel \emph{unnormalized} model inference method, Discriminative Likelihood Estimation
(DLE), by following the KL divergence minimization criterion. The algorithm uses
a technique called Density Ratio Estimation \citep{Sugiyama2012} which is
conventionally used to estimate the ratio between two density functions from
two sets of samples. We adapt this method to estimate the ratio between a data
and an unnormalized density model with the help of a Stein operator. We then use the
estimated ratio to construct a surrogate to KL divergence which is later
minimized to fit the parameters of an \emph{unnormalized} density function. The
resulting algorithm is a $\min \max$ problem, which we show can be conveniently converted into a min-min problem using Lagrangian duality. No extra
sampling steps are required. 

We further prove the consistency and asymptotic properties of DLE under mild conditions. One of our major contributions is that we prove the proposed estimator can also attain the asymptotic Cram{\'e}r-Rao bound. Numerical experiments validate our theories and we show DLE indeed performs well under realistic settings. 

\vspace*{-4mm}
\section{Background}
\subsection{Problem: Intractable Model Inference via KL Divergence Minimization}
\label{sec.back}
Consider the problem of estimating the parameter $\boldtheta$ of a
probability density model $p(\boldx; \boldtheta)$ from a set of i.i.d. samples: $X_q :=
\{\boldx^{(i)}_q\}_{i=1}^{n_q}\iid Q$ where $Q$ is a probability distribution whose density function is $q(\boldx)$. 
One idea is minimizing the sample approximated KL divergence from $p_\boldtheta$ to $q$:
\begin{align*}
\vspace*{-4mm}
	\min_\boldtheta \mathrm{KL}\left[q|p_\boldtheta\right] = 	\min_\boldtheta \mathbb{E}_q\left[\log \frac{q(\boldx)}{p(\boldx; \boldtheta)}\right] &=  C - \max_\boldtheta \mathbb{E}_q\left[\log p(\boldx; \boldtheta)\right] \notag \\
	&\approx C - \max_\boldtheta \frac{1}{n_q} \sum_{i=1}^{n_q} \log p(\boldx_q^{(i)};\boldtheta),
\end{align*}
where $C$ is a constant that does not depend on $\boldtheta$. The last line uses $X_q$ to approximate the expectation over $q(\boldx)$. This technique is known as Maximum Likelihood Estimation (MLE). 

Despite many advantages, MLE is unfit for intractable model inference. Consider
for instance a density model $p(\boldx; \boldtheta) :=
\frac{\bar{p}(\boldx;\boldtheta)}{z(\boldtheta)}$, where
$\bar{p}(\boldx;\boldtheta)$ is a positive multilayer neural network
parametrized by $\boldtheta$, $Z(\boldtheta) = \int
\bar{p}(\boldx;\boldtheta)\dx $ is the normalization term which guarantees that
$p(\boldx; \boldtheta)$ integrates to 1 over its domain. In this example,
$Z(\boldtheta)$ does not have a computationally tractable form; therefore, MLE
cannot be used without approximating the likelihood function or its gradient
using numerical methods such as Markov chain Monte Carlo (MCMC). 

However, there is an \emph{alternative approach} to minimizing the KL
divergence: $\mathrm{KL}\left[q|p_\boldtheta\right]$ is an expectation of the
log-ratio $\log\frac{q(\boldx)}{p(\boldx;\boldtheta)}$ with respect to the data
distribution $q(\boldx)$. If we have access to
$\frac{q(\boldx)}{p(\boldx;\boldtheta)}$, we can
approximate this KL by taking the average of the density ratio function over
samples $X_q$, and the density model parameter $\boldtheta$ can be subsequently
estimated by minimizing this approximation to the KL divergence.

\subsection{Two Sample Density Ratio Estimation}
Traditionally, Density Ratio Estimation (DRE) \citep{Sugiyama2008a,Sugiyama2012} refers to estimating the ratio of two unknown densities from their samples.
Given \emph{two} sets of i.i.d. samples drawn separately from distributions $Q$ and $P$:
$X_q := \{\boldx^{(i)}_q\}_{i=1}^{n_q}\sim Q, 
X_p := \{\boldx^{(i)}_p\}_{i=1}^{n_p}\sim P, \boldx_q, \boldx_p \in \mathbb{R}^d,
$ where distribution $Q$ and $P$ have density functions $q(\boldx)$ and $p(\boldx)$ respectively.
We hope to estimate the ratio $\frac{q(\boldx)}{p(\boldx)}$. 

We can model the density ratio using a function $r(\boldx; \bolddelta)$ parameterized by $\bolddelta$. To obtain the parameter $\bolddelta$, we minimize the KL divergence $\mathrm{KL}[q|q_\bolddelta]$
where $q(\boldx;\bolddelta):=r(\boldx;\bolddelta) p(\boldx)$:
\begin{align}
\label{eq.kliep}
\min_\bolddelta \mathrm{KL}[q|q_\bolddelta]
~~\text{s.t.}
	\textstyle\int r(\boldx;\bolddelta) p(\boldx) \dx = 1.
\end{align}
$\mathrm{KL}[q|q_\bolddelta]$ comprises three terms in which only one term is dependent on the parameter $\bolddelta$:
\begin{align}
\label{eq.sample.KL}
	\mathrm{KL}[q|q_\bolddelta] = \mathbbE_q[\log q(\boldx)] - \mathbbE_q[\log r(\boldx;\bolddelta)] - \mathbbE_q[\log p(\boldx)] 
	\approx -\textstyle\frac{1}{n_q} \sum_{i=1}^{n_q} \log r(\boldx^{(i)}_q;\bolddelta) + C,
\end{align}
The last step uses $X_q$ to approximate the expectation over $q(\boldx)$. $C$ is a constant irrelevant to $\bolddelta$. We can also approximate the equality constraint in \eqref{eq.kliep} using $X_p$:
\begin{align}
\label{eq.sample.normal}
\textstyle \int r(\boldx;\bolddelta) p(\boldx) \dx \approx \textstyle\frac{1}{n_p} \sum_{j=1}^{n_p} r(\boldx^{(j)}_p;\bolddelta).
\end{align}
Combining \eqref{eq.sample.KL} and \eqref{eq.sample.normal}, we get a sample version of \eqref{eq.kliep}:
\begin{align}
\label{eq.kliep.emp}
\hat{\bolddelta} :=\argmin_\bolddelta& -\textstyle\frac{1}{n_q} \sum_{i=1}^{n_q} \log r(\boldx^{(i)}_q;\bolddelta) + C 
~~~\text{ s.t.}
\textstyle\frac{1}{n_p} \sum_{j=1}^{n_p} r(\boldx^{(j)}_p;\bolddelta) = 1.
\end{align}
The above optimization is called Kullback Leibler Importance Estimation Procedure (KLIEP) \citep{Sugiyama2008a}.
Unfortunately, it \emph{cannot} be directly used to estimate our ratio
$\frac{q(\boldx)}{p(\boldx;\boldtheta)}$ since we only have samples from $q(\boldx)$ but not from $p(\boldx;\boldtheta)$. Consequently the equality constraint $\int r(\boldx;\bolddelta)
p(\boldx;\boldtheta) \dx = 1$ can no longer be approximated using samples. 

A natural remedy to this problem is to draw samples from
$p(\boldx;\boldtheta)$ using sampling techniques, such as MCMC which, in general, can be
costly when $p(\boldx;\boldtheta)$ is complex. 
Correlation among drawn samples from an MCMC scheme further complicates
estimation of the ratio. More importantly, regardless of the feasibility of
sampling from $p(\boldx;\boldtheta)$, the availability of an explicit (possibly unnormalized)
density $p(\boldx;\boldtheta)$ is much more valuable than just samples, especially in a high
dimensional space where samples rarely capture the fine-grained structural
information present in the density model $p(\boldx;\boldtheta)$.

In this work, we propose a new procedure -- Stein Density Ratio Estimation --
which can directly use the (unnormalized) density $p$, as it is, without
sampling from it. Moreover, the new procedure (described in Section
\ref{sec.stein.fea}) yields a  density ratio model
$r_{\boldtheta}(\boldx;\bolddelta)$ for the ratio function
$\frac{q(\boldx)}{p(\boldx;\boldtheta)}$ that automatically satisfies the
aforementioned equality constraint for all $\boldtheta$.

\section{Stein Density Ratio Estimation}
\label{sec.SDRE}
Let us consider a linear-in-parameter density ratio model $r(\boldx;\bolddelta)
:= \bolddelta^\top \boldf(\boldx)$, where $\boldf(\boldx)$ is a ``feature
function'' that transforms a data point $\boldx$ into a more powerful
representation. To better model
$\frac{q(\boldx)}{p(\boldx;\boldtheta)}$, we define a family of feature
functions called Stein features. 

\subsection{Stein Features}
\label{sec.stein.fea}
Suppose we have a feature function $\boldf(\boldx): \mathbb{R}^d \to \mathbb{R}^b$ and a density model $p(\boldx;\boldtheta): \mathbb{R}^{d}\to\mathbb{R}$. A Stein feature $T_{\boldtheta} \boldf(\boldx)\in \mathbb{R}^{b}$ with respect to $p(\boldx;\boldtheta)$ is 
$T_{\boldtheta} \boldf(\boldx) := [T_{\boldtheta} f_1(\boldx), \dots, T_{\boldtheta} f_i(\boldx), \dots, T_{\boldtheta} f_b(\boldx)]^\top,$ where $T_{\boldtheta}$ is a \emph{Stein operator} \citep{Stein1972,Gorham2015,Chwialkowski2016,oates2017control} and 
$T_{\boldtheta}f_i(\boldx) \in \mathbb{R}$ is defined as \[T_{\boldtheta} f_i(\boldx) := \langle \nabla_{\boldx} \log  p(\boldx;\boldtheta), \nabla_\boldx f_i(\boldx) \rangle +  \mathrm{trace}(\nabla^2_\boldx f_i(\boldx)),\]
where $f_i$ is the $i$-th output of function $\boldf$, $\nabla_\boldx f_i(\boldx)$ is the gradient of $f_i(\boldx)$ and $\nabla^2_\boldx f_i(\boldx)$ is the Hessian of $f_i(\boldx)$. 
Note that computing $T_{\boldtheta} \boldf(\boldx)$ does \emph{not} require evaluating the normalization term $Z(\boldtheta)$ as 
\[\nabla_{\boldx} \log p(\boldx;\boldtheta) = \nabla_{\boldx} \log \bar{p}(\boldx;\boldtheta) - \nabla_{\boldx}\log Z(\boldtheta),
\text{ where }\nabla_{\boldx} \log Z(\boldtheta)= 0.\]
\begin{example}
	\label{ex.expo}
	Let $p(\boldx;\boldtheta)$ be in exponential family with sufficient statistic $\boldpsi(\boldx)$, then \[T_{\boldtheta} f_i(\boldx)  =  \boldtheta^\top \boldJ_\boldx \boldpsi(\boldx) \nabla_\boldx f_i(\boldx) + \mathrm{trace}[\nabla^2_\boldx f_i(\boldx)],\]
where $\boldJ_\boldx \boldpsi(\boldx)\in \mathbb{R}^{\mathrm{dim}(\boldtheta) \times d}$ is the Jacobian of $\boldpsi(\boldx)$ and $\dim(\boldtheta)$ is the dimension of $\boldtheta$ . 
\end{example}
One more example can be found at Appendix, Section \ref{sec.stein.ex}.
A slightly different Stein operator was introduced in \citep{Chwialkowski2016,oates2017control} where
$T'_{\boldtheta}\boldf(\boldx) \in \mathbb{R}$ for $\boldf(\boldx) \in \mathbb{R}^d$ is defined as $T'_{\boldtheta} \boldf(\boldx) := \sum_{i=1}^{d} \left[\partial_{x_i} \log p(\boldx;\boldtheta)\right] f_i(\boldx) + \partial_{x_i} f_i(\boldx),$ where $\partial_{x_i} f(\boldx)$ is the partial derivative of $f(\boldx)$ with respect to $x_i$.
We can see the relationship between $T$ and $T'$: $T_{\boldtheta} f_i(\boldx) = T'_{\boldtheta} \nabla_\boldx f_i(\boldx)$.
Next we show an important property of Stein features. 
\begin{prop}[Stein's Identity]
	\label{prop.stein.equality}Suppose $p(\boldx;\boldtheta)>0$, \[\forall_{i,j} \lim_{|x_j|\to \infty} p(x_1,\cdots,x_j,\cdots,x_d;\boldtheta)\partial_{x_j}f_i(x_1,\cdots,x_j,\cdots,x_d) = 0,\] 
	$p(\boldx;\boldtheta)$ is continously differentiable and $f_i$ is twice continuously differentiable for all $\boldtheta$ and $i$.
Then 
$	\mathbbE_{p_\boldtheta} [T_{\boldtheta} \boldf(\boldx)] = \boldzero
$ for all $\boldtheta$.
\end{prop}
We give a proof in Appendix Section \ref{sec.proof.stein.eq}. 
Similar identities were given in previous literatures such as Lemma 2.2 in \citep{LiuQ2016} or Lemma 5.1 in \citep{Chwialkowski2016}.
Utilizing this property, we can construct a density ratio model which bypasses the ``intractable equality constraint''
issue when estimating $\frac{q(\boldx)}{p(\boldx;\boldtheta)}$ as shown in the next section.

\subsection{Stein Density Ratio Modeling and Estimation (SDRE)}
\label{sec.stein.dre}
Define a linear-in-parameter density ratio model:
$
	r_\boldtheta(\boldx; \bolddelta) := \bolddelta^\top T_{\boldtheta} \boldf(\boldx) + 1
$
by using a Stein feature function. We can see that $\mathbb{E}_{p_\boldtheta} [r_\boldtheta(\boldx; \bolddelta)] = \mathbb{E}_{p_{\boldtheta}}[\bolddelta^\top T_{\boldtheta} \boldf(\boldx) + 1] = \bolddelta^\top\mathbb{E}_{p_{\boldtheta}}[ T_{\boldtheta} \boldf(\boldx)] +1 = 1$ where the last equality is ensured by Proposition \ref{prop.stein.equality} for all $\bolddelta$ and $\boldtheta$ if the specified regularity conditions are met. This equality means the constraint in \eqref{eq.kliep} is automatically satisfied with this density ratio model.
Now we can solve \eqref{eq.kliep.emp} without its equality constraint.
\begin{align}
	\label{eq.steinDR}
	\hat{\bolddelta} :=& \argmin_{\bolddelta\in\mathbb{R}^{b}} -\frac{1}{n_q} \sum_{i=1}^{n_q} \log r_\boldtheta(\boldx_q^{(i)};\bolddelta)  + C
	= \argmax_{\bolddelta\in\mathbb{R}^{b}} \frac{1}{n_q} \sum_{i=1}^{n_q} \log \left[\bolddelta^\top T_{\boldtheta} \boldf(\boldx_q^{(i)}) + 1 \right].
\end{align}  
It can be seen that \eqref{eq.steinDR} is an unconstrained \emph{concave} maximization problem. Note for all $\boldx_q \in X_q$, $r_\boldtheta(\boldx_q;\hat{\bolddelta})$ \emph{must be strictly positive} thanks to the \emph{log-barrier} (see e.g., Section 17.2 in \citep{NoceWrig06}) in our objective function. However, it is \emph{not} possible to guarantee that for all $\boldx \in \mathbb{R}^d$, $r_\boldtheta(\boldx;\hat{\bolddelta})$ is positive. This is not a problem in this paper, as the density ratio function is only used for approximating the KL divergence, and we will not evaluate $r_\boldtheta(\boldx;\hat{\bolddelta})$ at a data point $\boldx$ that is outside of $X_q$. Note, the unnormalized density model $\bar{p}(\boldx;\boldtheta)$, by definition, should be non-negative everywhere for all $\boldtheta$. 

We refer to the objective \eqref{eq.steinDR} as Stein Density Ratio Estimation (SDRE). 
One may notice that $\frac{1}{n_q} \sum_{i=1}^{n_q} \log r_\boldtheta(\boldx_q^{(i)};\hat{\bolddelta})$ evaluated at $\hat{\bolddelta}$ is exactly the sample average of the estimated ratio over  $X_q$ which allows us to approximate the KL divergence from $p(\boldx;\boldtheta)$ to $q(\boldx)$.

\section{Intractable Model Inference via Discriminative Likelihood Estimation}
\label{sec.parafitting}
Let $\ell(\hat{\bolddelta}, \boldtheta) := \frac{1}{n_q} \sum_{i=1}^{n_q} \log r_\boldtheta(\boldx_q^{(i)};\hat{\bolddelta})$. We will use $\ell(\hat{\bolddelta}, \boldtheta)$ as a replacement of $\mathrm{KL}\left[q(\boldx)|p(\boldx;\boldtheta) \right]$. The rationale of minimizing KL divergence from $p(\boldx;\boldtheta)$ to $q(\boldx)$ leads to:
\begin{align}
\label{eq.minmax.fitting}
	\min_\boldtheta \ell(\hat{\bolddelta}, \boldtheta)
	\text{ or equivalently } \min_\boldtheta \max_\bolddelta \ell(\bolddelta, \boldtheta).
\end{align}
The equivalence is due to the fact that $\ell$ evaluated at the optimal ratio parameter $\hat{\bolddelta}$ is also the maximum of the DRE objective function when being optimized w.r.t. $\bolddelta$.
The outer problem minimizes $\ell$ with respect to the density parameter $\boldtheta$.  We call this estimator \textbf{Discriminative Likelihood Estimation} (DLE) as the parameter of the density model $p(\boldx; \boldtheta)$ is learned via minimizing a \emph{discriminator}\footnote{The word ``discriminator'' is borrowed from GAN \citep{Goodfellow2014}. Indeed, DLE and GAN bears many resemblances.  }, which is the likelihood ratio function $\ell(\hat{\bolddelta},\boldtheta)$ measuring the differences between $q(\boldx)$ and $p(\boldx;\boldtheta)$.

\subsection{Consistency with Correct Model}
For brevity, we state all theorems assuming all regularity conditions in Proposition \ref{prop.stein.equality} are met.

\paragraph{Notations:}
$\boldH$ is $\nabla_{(\bolddelta, \boldtheta)}^2\ell(\bolddelta,\boldtheta)$, the full Hessian of  $\ell(\bolddelta,\boldtheta)$. 
$\boldH_{\bolddelta, \boldtheta}$ is $\nabla_\bolddelta\nabla_\boldtheta \ell(\bolddelta, \boldtheta)$, submatrix of the Hessian matrix whose rows and columns indexed by $\bolddelta, \boldtheta$ respectively.  $\bolds\in \mathbb{R}^{\mathrm{dim}(\boldtheta)}$ is $\nabla_\boldtheta \log p(\boldx,\boldtheta)$ evaluated at $\boldtheta^*$, score function of $p(\boldx;\boldtheta)$. $\lambda(\cdot)$ is the eigenvalue operator. $\lambda_\mathrm{min}(\cdot)$ or $\lambda_\mathrm{max}(\cdot)$ is the minimum or maximum eigenvalue and $\|\cdot\|$ is the operator norm.

We study the consistency of the following  estimator under a correct model setting.
\begin{align}
\label{eq.thm.obj}
(\hat{\bolddelta}, \hat{\boldtheta}) := \arg \min_{\boldtheta \in \Theta}  \max_{\bolddelta \in \Delta_{n_q}} \ell(\bolddelta, \boldtheta),
\end{align}	
where $\Theta$ and $\Delta_{n_q}$ are \emph{compact} parameter spaces for $\boldtheta$ and $\bolddelta$ respectively. 
The compactness condition is among a set of conditions commonly used in classic consistency proofs (see e.g., Wald's Consistency Proof, 5.2.1,\citep{vaart_1998}). 
It is possible to derive weaker conditions given specific choices of $\boldf$ or $p(\boldx;\boldtheta)$. However, in the current manuscript, we only focus on more generic settings and conditions that would give rise to estimation consistency and useful asymptotic theories. 
We assume they are properly chosen so that $(\hat{\boldtheta}, \hat{\bolddelta})$ is the saddle point of \eqref{eq.thm.obj}. 

First, we assume our density model $p(\boldx;\boldtheta)$ is correctly specified: 
\begin{assumption}
	\label{eq.model}
	There exists a unique pair of parameter $(\boldtheta^*, \bolddelta^*), \boldtheta^*\in \Theta$, $\bolddelta^* \in \Delta_{n_q}$, such that $p(\boldx; \boldtheta^*)  \equiv q(\boldx)$ and $r_{\boldtheta^*}(\boldx;\bolddelta^*) = 1$.
\end{assumption}
Given how $r_\boldtheta(\boldx;\bolddelta)$ is constructed in \Cref{sec.stein.dre}, the above assumption implies $\bolddelta^*$ must be $\boldzero$.
\begin{assumption}
	\label{assum.info.hessian}
	There exist constants $\Lambda_\mathrm{min}>0, \Lambda'_\mathrm{min}>0$ and $\Lambda_\mathrm{max}>0$ so that $\forall \boldtheta \in \Theta, \bolddelta \in \Delta_{n_q}$
	\begin{align*}
	&\lambda_\mathrm{min}\left\{-\boldH_{\bolddelta,\bolddelta}\right\} \ge \Lambda'_\mathrm{min}, 	\Lambda_\mathrm{max}\ge\left\|\boldH_{\boldtheta,\bolddelta}\boldH_{\bolddelta,\bolddelta}^{-1}\right\|, \lambda_\mathrm{min}\left\{-\boldH_{\boldtheta, \bolddelta}\boldH_{\bolddelta,\bolddelta}^{-1}\boldH_{\bolddelta,\boldtheta}\right\} \ge \Lambda_\mathrm{min} > 2\left\|\boldH_{\boldtheta,\boldtheta}\right\|.
	\end{align*}
\end{assumption}
The lower-boundedness of $\lambda_\mathrm{min}\left\{-\boldH_{\bolddelta,\bolddelta}\right\}$ implies the \emph{strict} concavity of $\ell(\bolddelta, \boldtheta)$ with respect to $\bolddelta$ ($\ell(\bolddelta, \boldtheta)$ is already concave by construction, see \eqref{eq.steinDR}): For all $\boldtheta \in \Theta$, there exists a unique $\hat{\bolddelta}(\boldtheta)$ that maximizes the likelihood ratio, which means the likelihood ratio function should always have sufficient discriminative power to precisely pinpoint the differences between our data and the current model $\boldtheta$.
It also ensures that $\bolddelta$ can ``teach'' the model parameter $\boldtheta$ well by assuming the ``interaction'' between $\bolddelta$ and $\boldtheta$ in our estimator, $\boldH_{\boldtheta,\bolddelta}$, is well-behaved.

Now we analyze Assumption \ref{assum.info.hessian} on a special case: 
\begin{prop}
\label{prop.ass2.expo}
Let $\Delta_{n_q} := \left\{\bolddelta \bigg\rvert \frac{1}{C_\mathrm{ratio}} \le r_\boldtheta(\boldx;\bolddelta) \le C_\mathrm{ratio}, \| \bolddelta \|_2 \le T/\sigma(n_q), \forall \boldtheta \in \Theta, \forall \boldx \in X_q,  \right\}$ where $T>0, C_\mathrm{ratio}>1$ are constants and $\sigma(\cdot)$ is a monotone-increasing function. $p(\boldx;\boldtheta)$ is in exponential family with sufficient statistic $\boldpsi(\boldx)$ and Stein feature is chosen as $T_{\boldtheta}\boldpsi(\boldx)$. Suppose there exist constants $ C_2, C_3, C_4, C_5, \Lambda''_\mathrm{max},\Lambda''_\mathrm{min} >0, C_2 \ge \frac{1}{n_q} \sum_{i=1}^{n_q} \|\boldJ_\boldx \boldpsi(\boldx_q^{(i)})\|^4,$
\begin{align*}
	&\lambda_{\mathrm{min}}\left\{ \frac{1}{n_q} \sum_{i=1}^{n_q} \boldJ_\boldx \boldpsi(\boldx_q^{(i)}) \boldJ_\boldx \boldpsi(\boldx_q^{(i)})^\top\right\} \ge C_3,
	\frac{1}{n_q} \sum_{i=1}^{n_q} \|\boldJ_\boldx \boldpsi(\boldx_q^{(i)})
	{\boldJ_\boldx
	\boldpsi(\boldx_q^{(i)})}^\top\|\le C_4,\\
	&\frac{1}{n_q} \sum_{i=1}^{n_q} \|\boldJ_\boldx \boldpsi(\boldx_q^{(i)})
	{\boldJ_\boldx
	\boldpsi(\boldx_q^{(i)})}^\top\|\cdot \|T_\boldtheta \bold\boldpsi(\boldx_q^{(i)}) \|\le C_5
\end{align*}
 and 
$\Lambda''_\mathrm{max} \ge \lambda \left( \frac{1}{n_q} \sum_{i=1}^{n_q} T_{\boldtheta}\boldpsi(\boldx_q^{(i)}) T_{\boldtheta}\boldpsi(\boldx_q^{(i)})^\top\right) \ge \Lambda''_{\mathrm{min}}, \forall \boldtheta\in \Theta$ with high probability. There exists a constant $N>1$, when $n_q \ge N$, Assumption \ref{assum.info.hessian} holds with high probability.  
\end{prop}
The proof can be found in Appendix, Section \ref{sec.proof.ass2.expo}.
Note in practice the domain constraint of $\Delta_{n_q}$ in this proposition can be easily enforced via convex constraints or penalty terms. Analysis on a few other examples can be found in Appendix, Section \ref{sec.exp.ass2}.

Proposition \ref{prop.ass2.expo} gives us some hints on how the feature function $\boldf$ of Stein feature can be chosen. In the case of exponential family, the choice $\boldf = \boldphi$ guarantees Assumption \ref{assum.info.hessian} to hold with high probability when $n_q$ increases. 

\begin{assumption}[Concentration of Stein features]
	The difference between the sample average of the Stein feature $	T_{\boldtheta^*}\boldf(\boldx)$ and its expectation over $q$ converges to zero in $\ell_2$ norm in probability.
	\label{ass.concentration}
	$
	\left\|\frac{1}{n_q}\sum_{i=1}^{n_q}
	T_{\boldtheta^*}\boldf(\boldx_q^{(i)})
	-  \mathbbE_{q}\left[T_{\boldtheta^*}\boldf(\boldx)\right]\right\|_2 \overset{\mathbb{P}}{\to} 0.$ 
\end{assumption}
Note, if Assumption \ref{eq.model} holds at the same time, Proposition \ref{prop.stein.equality} indicates $\mathbbE_{q}\left[T_{\boldtheta^*}\boldf(\boldx)\right] \equiv \boldzero$.
This assumption holds due to the (strong) law of large numbers given that the $\mathbbE_{q}\left[T_{\boldtheta^*}\boldf(\boldx)\right]$ exists.

\begin{thm}[Consistency]
\label{thm.main}
Suppose Assumption \ref{eq.model},  \ref{assum.info.hessian} and \ref{ass.concentration} holds, $(\hat{\bolddelta}, \hat{\boldtheta}) \overset{\mathbb{P}}{\to} (\boldzero, \boldtheta^*)$.
\end{thm}
See Section \ref{sec.main.proof} in Appendix for the proof.
This theorem states that as $n_q$ increases, all saddle points of
\eqref{eq.thm.obj},  converge to the vicinity of true parameters. 
All the following theorems rely on the result of Theorem \ref{thm.main}.

\subsection{Asymptotic Variance of $\hat{\boldtheta}$ and Fisher Efficiency of DLE}
In this section we state one of our main contributions: 
DLE can attain the efficiency bound, i.e., asymptotic Cram{\'e}r-Rao bound when $\boldf(\boldx)$ is appropriately chosen.
First, we show our estimator $\hat{\boldtheta}$ has a simple asymptotic distribution which allows us to perform model inference. To state the theorem, we need an extra assumption on the Hessian $\boldH$: 
\begin{assumption}[Uniform Convergence on $\boldH$] 
	\label{assum.unifor.conv}
$	\sup_{\bolddelta \in \Delta_{n_q}, \boldtheta \in \Theta}
	\left|H_{i,j} - \mathbb{E}_q\left[H_{i,j}\right]\right| \overset{\mathbb{P}}{\to} 0, \forall_{i,j}.
$\end{assumption}
This assumption states the second order derivatives (which is an average over samples from $X_q$)
converges \textbf{uniformly} to its population mean, as $n_q \to \infty$. It
helps us control the residual in the second order Taylor expansion in our proof. 
This assumption may be weakened given specific choices of $\boldf$
and $p(\boldx;\boldtheta)$ but we focus on establishing the asymptotic results  in generic settings, so this condition is only listed as an
assumption. 

\begin{thm}[Asymptotic Normality of $\hat{\boldtheta}$]
	\label{thm.normality2}
	Suppose Assumption \ref{eq.model},  \ref{assum.info.hessian}, \ref{ass.concentration} and \ref{assum.unifor.conv} holds, 
	\begin{align}
	\label{eq.asym.variance}
	\sqrt{n_q} \left(\boldtheta^* - \hat{\boldtheta} \right) \rightsquigarrow 
	\mathcal{N}\left[0, \boldV \right],
	\boldV = \left(-\mathbb{E}_q\left[\boldH^*_{\boldtheta,\bolddelta}\right] {\mathbb{E}_q\left[\boldH^*_{\bolddelta,\bolddelta}\right]}^{-1}{\mathbb{E}_q\left[\boldH^*_{\bolddelta,\boldtheta}\right]}\right)^{-1},
	\end{align}
	where  $\boldH^*$ is $\boldH$ evaluated at $(\bolddelta^*,\boldtheta^*)$.
\end{thm}
See Section \ref{sec.norm.proof} in Appendix for the proof.
In practice, we do not know $\mathbb{E}_q\left[\boldH^*\right]$, so we may use $\hat{\boldH}$, the Hessian of $\ell(\bolddelta,\boldtheta)$ evaluated at $(\hat{\bolddelta}, \hat{\boldtheta})$ as an approximation to $\mathbb{E}_q\left[\boldH^*\right]$.

Although MLE is also asymptotically normal, important quantities such as Fisher Information Matrix may not be efficiently computed on intractable models. 
In comparison, Theorem \ref{thm.normality2} enables us to compute parameter confidence interval for DLE even on intractable $p_\boldtheta$. 

Now we consider the asymptotic efficiency of the DLE with respect to specific choices of Stein features. 	
Let $\boldV_{\boldf}$ be the asymptotic variance \eqref{eq.asym.variance} using a  Stein feature with a specific choice of $\boldf$.

\begin{lem}
	\label{cor.mono}
	Suppose
	that Assumptions \ref{eq.model},  \ref{assum.info.hessian},  \ref{ass.concentration} and \ref{assum.unifor.conv} hold and 
	$\mathbb{E}_q[T_{\boldtheta^*}\boldf(\boldx) T_{\boldtheta^*}\boldf(\boldx)^{\top}]$ is invertible. 
	Moreover, suppose that the integration and the derivative of
	$\partial_{\theta_i}\int p(\boldx;\boldtheta)  T_{\boldtheta}\boldf(\boldx)\dx$
	is exchangeable for all $i$. 
	$V_{\boldf}
	=
	\left(
	\mathbb{E}_q[\bolds T_{\boldtheta^*}\boldf(\boldx)^{\top}]  \mathbb{E}_q[T_{\boldtheta^*}\boldf(\boldx) T_{\boldtheta^*}\boldf(\boldx)^{\top}]^{-1}  \mathbb{E}_q[T_{\boldtheta^*}\boldf(\boldx)\bolds^{\top}]
	\right)^{-1}.$
\end{lem}
The proof is given in Section
\ref{cor.mono.proof} in the Appendix.
Lemma~\ref{cor.mono} expresses asymptotic variance using score function and Stein feature and is used to prove that the variance monotonically decreases as 
the vector space spanned by the Stein feature vectors  becomes larger.
\begin{cor}[Monotonocity of Asymptotic Variance]
    Define the inner product as $\mathbb{E}_q[f g]$ for functions $f$ and $g$.
	Let
	$T_{\boldtheta^*}\boldf(\boldx)=[t_1,\dots,t_{b}]$ and 
	$T_{\boldtheta^*}\bar{\boldf}(\boldx)=[\bar{t}_1,\dots,\bar{t}_{\bar{b}}]$ be
	two Stein feature vectors.
	Assume that
	$\mathrm{span}\{t_1,\ldots,t_{b}\} \subset\mathrm{span}\{\bar{t}_1,\ldots,\bar{t}_{\bar{b}}\}$, 
	where $\mathrm{span}\{\cdots\}$ denotes the linear space spanned by the specified elements. 
	Then, the inequality 
	$V_{\bar{\boldf}}\preceq V_{\boldf}$ holds in the sense of the positive definiteness. 
\end{cor}
\begin{proof}
	Let us define $P_{\boldf}\bolds$ as the orthogonal projection of $\bolds$ onto 
	$\mathrm{span}\{t_1,\ldots,t_{b}\}$. 
	A simple calculation yields
	$P_{\boldf}\bolds=\mathbb{E}_q[\bolds T_{\boldtheta^*}\boldf(\boldx)^{\top}]\mathbb{E}_q[T_{\boldtheta^*}\boldf(\boldx)T_{\boldtheta^*}\boldf(\boldx)^{\top}]^{-1} T_{\boldtheta^*}\boldf(\boldx)$, and thus, 
	Lemma~\ref{cor.mono} leads to 
	$V_{\boldf}^{-1}=\mathbb{E}_q[P_{\boldf}\bolds(P_{\boldf}\bolds)^{\top} ]$. 
	From the property of the orthogonal projection (see e.g., Theorem 2.23 in \citep{yanai2011}), we have 
	$\mathbb{E}_q[P_{\bar{\boldf}}\bolds(P_{\bar{\boldf}}\bolds)^{\top}] \succeq\mathbb{E}_q[P_{\boldf}\bolds(P_{\boldf}\bolds)^{\top}]$. 
	Therefore, we obtain $V_{\bar{\boldf}}\preceq V_{\boldf}$. 
\end{proof}
For $Q_{\boldf}\bolds=\bolds-P_{\boldf}\bolds$, we have
$
\mathbb{E}_q[\bolds\bolds^\top] 
=
\mathbb{E}_q[P_{\boldf}\bolds (P_{\boldf}\bolds)^{\top}]+\mathbb{E}_q[Q_{\boldf}\bolds(Q_{\boldf}\bolds)^{\top} ] = 
V_{\boldf}^{-1}+\mathbb{E}_q[Q_{\boldf}\bolds(Q_{\boldf}\bolds)^{\top}]. 
$
Thus, we see that the asymptotic variance converges to the inverse of the Fisher information, 
$\mathbb{E}_q[\bolds\bolds^\top]^{-1}$, as $P_{\boldf}\bolds$ gets close to $\bolds$. 
In particular, when the linear space $\mathrm{span}\{t_1,\ldots,t_{b}\}$ includes $\bolds$, 
$Q_{\boldf}\bolds$ vanishes and consequently the DLE with $\boldf(\boldx)$ is asymptotically efficient. 
\begin{example}
	Let $p(\boldx; \boldtheta)$ be the model of
	the $d$-dimensional multivariate Gaussian distribution $\mathcal{N}(\boldtheta,\mathrm{Iden}\cdot \sigma^2)$, where $\mathrm{Iden}$ is the identity matrix.
	Here the variance $\sigma^2$ is assumed to be known. 
	The score function is $s_j(\boldx;\boldtheta)=-(x_j-\theta_j)/\sigma^2$, and 
	the Stein feature vector defined from $f(\boldx)= \boldx$ 
	is $(T_{\boldtheta}\boldx)_j=-(x_j-\theta_j)/\sigma^2$ for $j=1,\ldots,d$. 
	Clearly, the score function is included in $\mathrm{span}\{t_1,\ldots,t_d\}$. 
	Hence, the DLE with $\boldf$ achieves the efficiency bound of the parameter estimation. 
\end{example}

One more example can be found in Appendix, Section \ref{sec.exp.asym}.
In fact, Corollary \ref{cor.mono} suggests that as long as we can represent the
score function $\bolds$ using Stein feature  $T_\boldtheta \boldf$ up to a
linear transformation, DLE can achieve efficiency bound. However, since
$\boldf$ is coupled with $\nabla_\boldx \log p(\boldx;\boldtheta)$ in
$T_\boldtheta \boldf$, it is not always easy to reverse engineer an $\boldf$
from $\bolds$. Nonetheless, our numerical experiments show that using simple
functions such polynomials as $\boldf$ yields good performance. 

\subsection{Model Selection of DLE}
\label{sec.msel}
As our objective \eqref{eq.minmax.fitting} tries to minimize the discrepancy between our model $p(\boldx;\boldtheta)$ and the data distribution, it is tempting to compare models using the objective function evaluated at $(\hat{\bolddelta},\hat{\boldtheta})$, i.e., $\ell(\hat{\bolddelta}, \hat{\boldtheta})$.
However, the more sophisticated $p(\boldx;\boldtheta)$ becomes, the more likely it picks up spurious patterns of our dataset. 
Similarly, the more powerful the Stein features are, the more likely the discriminator is overly critical to the density model on this dataset. 
Thus a better model selection criterion would be comparing $\mathbb{E}_q \left[\ell(\hat{\bolddelta}, \hat{\boldtheta})\right]$ which eliminates the potential of overfitting a specific dataset. Unfortunately, this expectation cannot be computed without the knowledge on $q(\boldx)$. We propose to approximate this quantity using a \emph{penalized likelihood}:
\begin{thm}
	\label{thm.aic.2}
	Suppose Assumption \ref{eq.model}, \ref{assum.info.hessian}, \ref{ass.concentration} and \ref{assum.unifor.conv} holds. $\mathbb{E}_q\left[\boldH^*_{\bolddelta,\bolddelta}\right]$ and $ \mathbb{E}_q\left[\boldH^*_{\bolddelta,\boldtheta}\right]$ are full-rank and $\mathrm{dim}(\boldtheta) \le b$, then
	$
	n_q\mathbb{E}_q \left[\ell(\hat{\bolddelta}, \hat{\boldtheta})\right] = \min_\boldtheta \max_\bolddelta n_q\ell(\bolddelta, \boldtheta)  -
	b + \mathrm{dim}(\boldtheta) + o_p(1).
	$
\end{thm}

See Section \ref{sec.aic2.proof} in Appendix for the proof. This theorem is closely related to a classic result called Akaike Information Criterion (AIC) \citep{Akaike1974}. Both AIC and \Cref{thm.aic.2} similarly penalize the degree of freedom of the density model $\mathrm{dim}(\boldtheta)$, while our theorem also penalizes the number of ratio parameter $\mathrm{dim}(\bolddelta) = b$ due to the fact that our ratio function is also fitted using samples. 

One can also show $\ell(\hat{\bolddelta},\hat{\boldtheta})$ follows a $\chi^2$ distribution. See Section \ref{sec.gof.proof} in Appendix for details. 

Theorem \ref{thm.aic.2} provides an information-criterion based model selection method. 
Suppose $M$ is a set of different Stein features and $M'$ is a set of candidate density models. 
We can jointly select density model and Stein feature:
$(\hat{m},\hat{m}'):=\arg\min_{m'\in M'}\max_{m \in M} \mathbbE_q[ \ell(\hat{\boldtheta}(m'), \hat{\bolddelta}(m)) ]$, where $(\hat{\boldtheta}(m'), \hat{\bolddelta}(m))$ are estimated parameters under the model choice $(m',m)$. Replacing $\mathbbE_q[ \ell(\hat{\boldtheta}(m'), \hat{\bolddelta}(m))]$ with the penalized likelihood derived in Theorem \ref{thm.aic.2}, we can get a practical model selection method. 

\section{Lagrangian Dual of SDRE and DLE by Minimization}
\label{sec.lag}
Some techniques can be used to directly optimize the min-max problem in \eqref{eq.minmax.fitting}, such as performing gradient descend/ascend on $\boldtheta$ and $\bolddelta$ alternately. However, looking for the saddle points of a min-max optimization is hard. In this section, we derive a partial Lagrangian dual for \eqref{eq.minmax.fitting} so we can convert this min-max problem into a min-min problem whose local optima can be efficiently found by existing optimization techniques. 

\begin{prop} 
\label{prop.lag.dual}
SDRE problem in \eqref{eq.steinDR} has a Lagrangian dual:
\begin{align}
\label{eq.lag.0}
    \hat{\boldmu} = \argmin_{\boldmu} \sum_{i=1}^{n_q} [- (\log -\mu_i) -1] - \sum_{i=1}^{n_q} \mu_i \mathrm{ ~~s.t. : }  \sum_{i=1}^{n_q} \mu_i T_{\boldtheta} \boldf(\boldx_q^{(i)})  = \boldzero.
\end{align}
Moreover, the duality gap between \eqref{eq.lag.0} and \eqref{eq.steinDR} is 0
and $r_\boldtheta(\boldx_q^{(i)};\hat{\bolddelta}) = -1/\hat{\mu}_i$.
\end{prop}
See Section \ref{sec.proof.dual} in the Appendix for its proof. 
Instead of solving the min-max problem \eqref{eq.minmax.fitting}, we solve the following constrained minimization problem:
\begin{align}
\label{eq.dual2}
\min_{\boldtheta} \min_{\boldmu} \sum_{i=1}^{n_q} [- (\log -\mu_i) -1] - \sum_{i=1}^{n_q} \mu_i, \mathrm{ ~~s.t. : }  \sum_{i=1}^{n_q} \mu_i T_{\boldtheta} \boldf(\boldx_q^{(i)})  = \boldzero,
\end{align}
where we replace the inner $\max$ problem in \eqref{eq.minmax.fitting} with its Lagrangian \eqref{eq.lag.0}.
All experiments in this paper are performed using the Lagrangian dual objective 
\eqref{eq.dual2}. 
See \url{https://github.com/lamfeeling/Stein-Density-Ratio-Estimation} for code demos on SDRE and model inference.
\vspace*{-4mm}

\section{Related Works}
\label{sec.related}
\textbf{Score Matching (SM) \citep{Hyvaerinen2005,hyvarinen2007}} is a inference method for unnormalized statistical models. It estimates model parameters by minimizing the Fisher divergence \citep{Lyu2009,Sanchez2012} between the true log density and the model log density. To estimate the parameter, this method only requires $\nabla_\boldx \log p(\boldx;\boldtheta)$ and $\nabla_\boldx^2\log p(\boldx;\boldtheta)$ to avoid evaluating the normalization term. 

\textbf{Kernel Stein Discrepancy (KSD) \citep{Barp2019}} is a kernel mean discrepancy measure between a data distribution and a model density using the Stein identity defined on Stein operator $T'_{p_{\boldtheta}}$. This measure has been used for model evaluation \citep{Chwialkowski2016,LiuQ2016}. In \Cref{sec.experiment}, we minimize such a discrepancy with respect to $\boldtheta$ for unnormalized model parameter estimation. A more generic version of this estimator has been discussed in \citep{Barp2019}.

\textbf{Noise Contrastive Estimation (NCE) \citep{gutmann2010noise}} estimates the parameters of an unnormalized statistical model by performing a non-linear logistic regression to discriminate between observed dataset and artificially generated noise. The normalization term can be dealt with like a regular parameter and estimated by such a logistic regression. NCE requires us to select a noise distribution and in our experiments, we use a multivariate Gaussian distribution with mean and variance the same as $X_q$. 

\section{Experiments}
\label{sec.experiment}
\subsection{Validation of Asymptotic Results}
\begin{figure}
    \centering
    \subfloat[\small qqplot of marginals]{
        \label{fig.asy.qqp}
        \includegraphics[width=.24\textwidth]{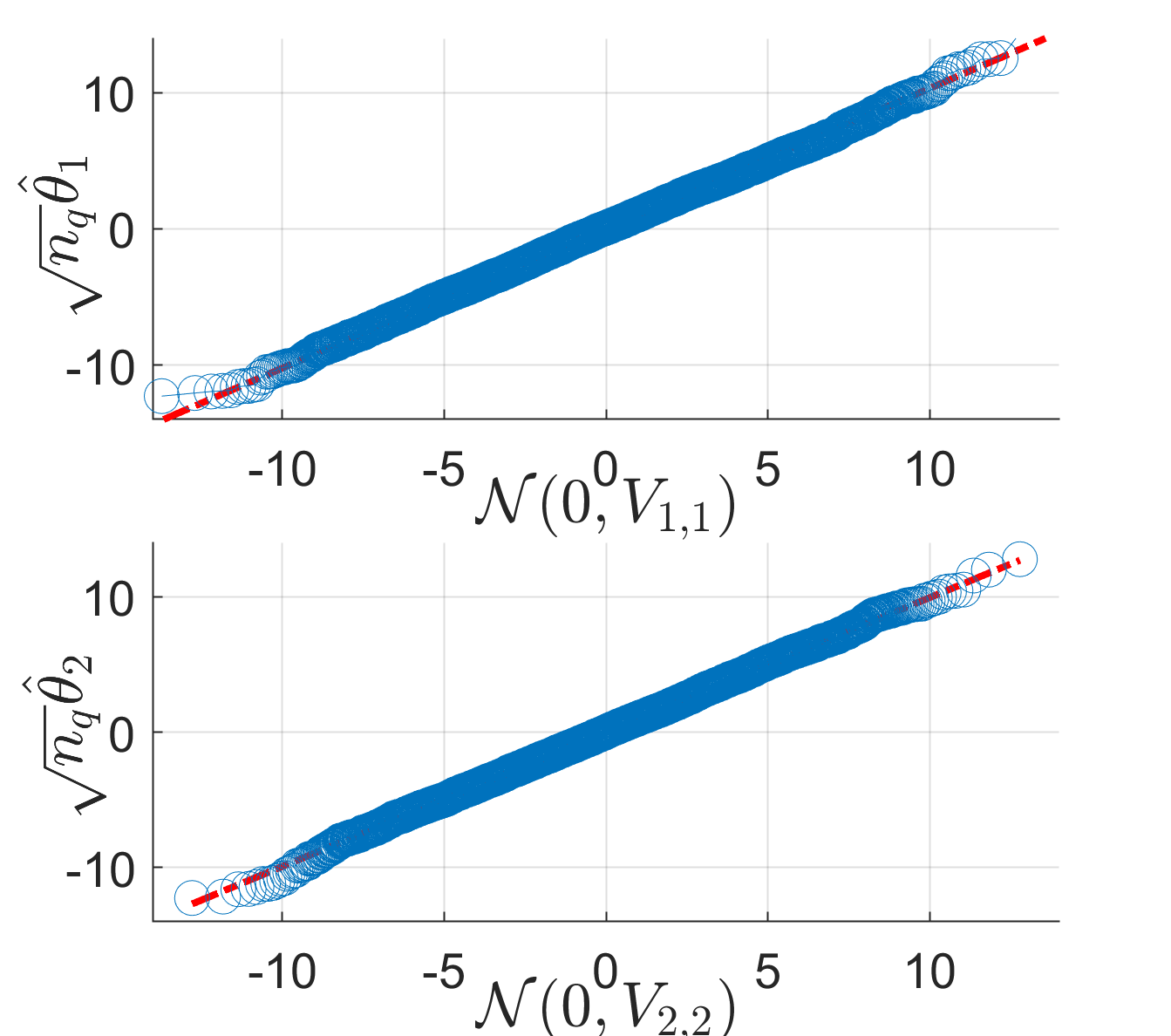}}
    \subfloat[\small $\sqrt{n_q} \hat{\boldtheta}$ vs. C.I.]{
        \label{fig.asy.theta}
        \includegraphics[width=.24\textwidth]{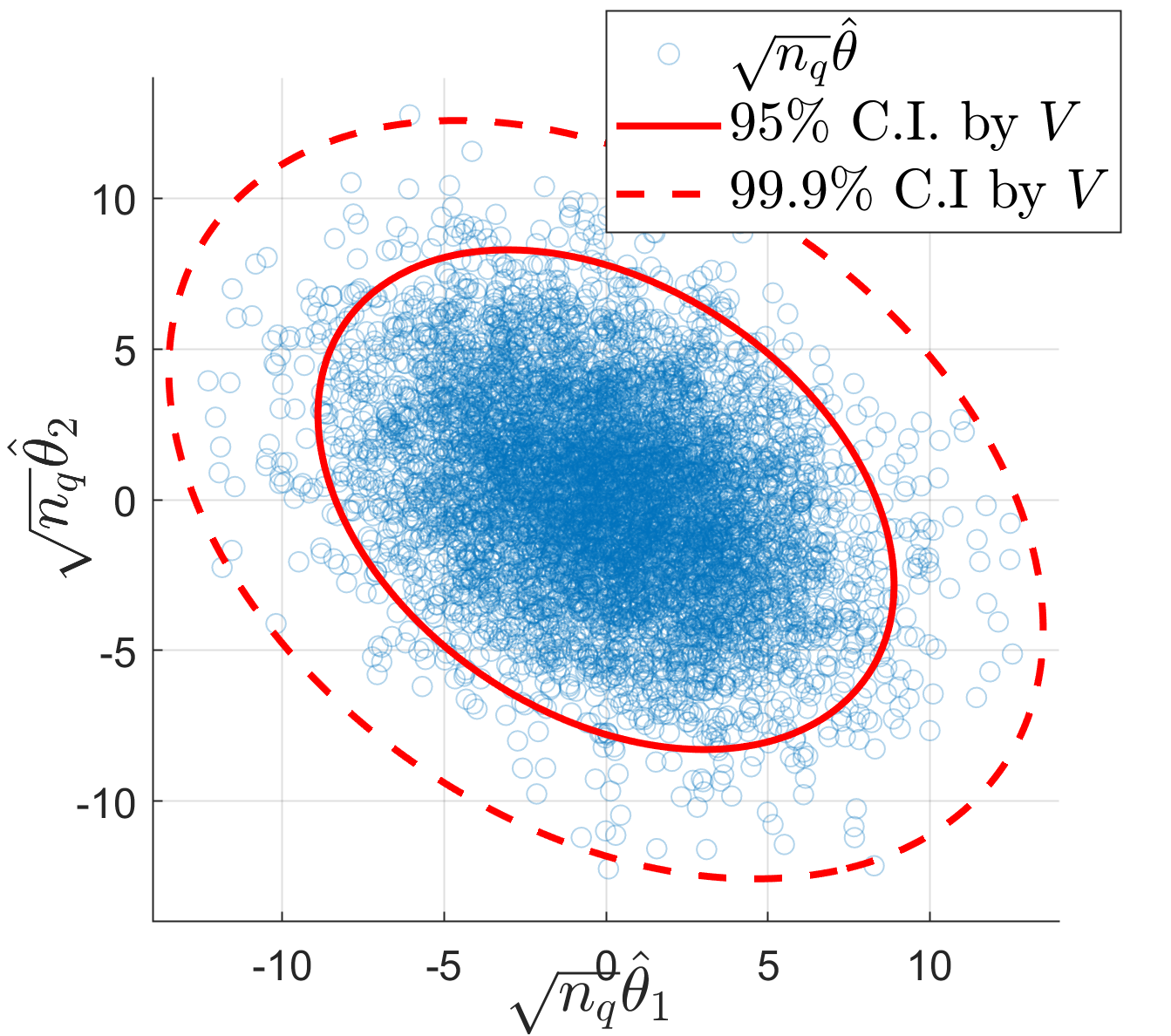}}
    \subfloat[\small$\mathrm{Var}(\hat{\boldtheta})$, Gamma dist.]{
        \label{fig.gamma.var}
        \includegraphics[width=.24\textwidth]{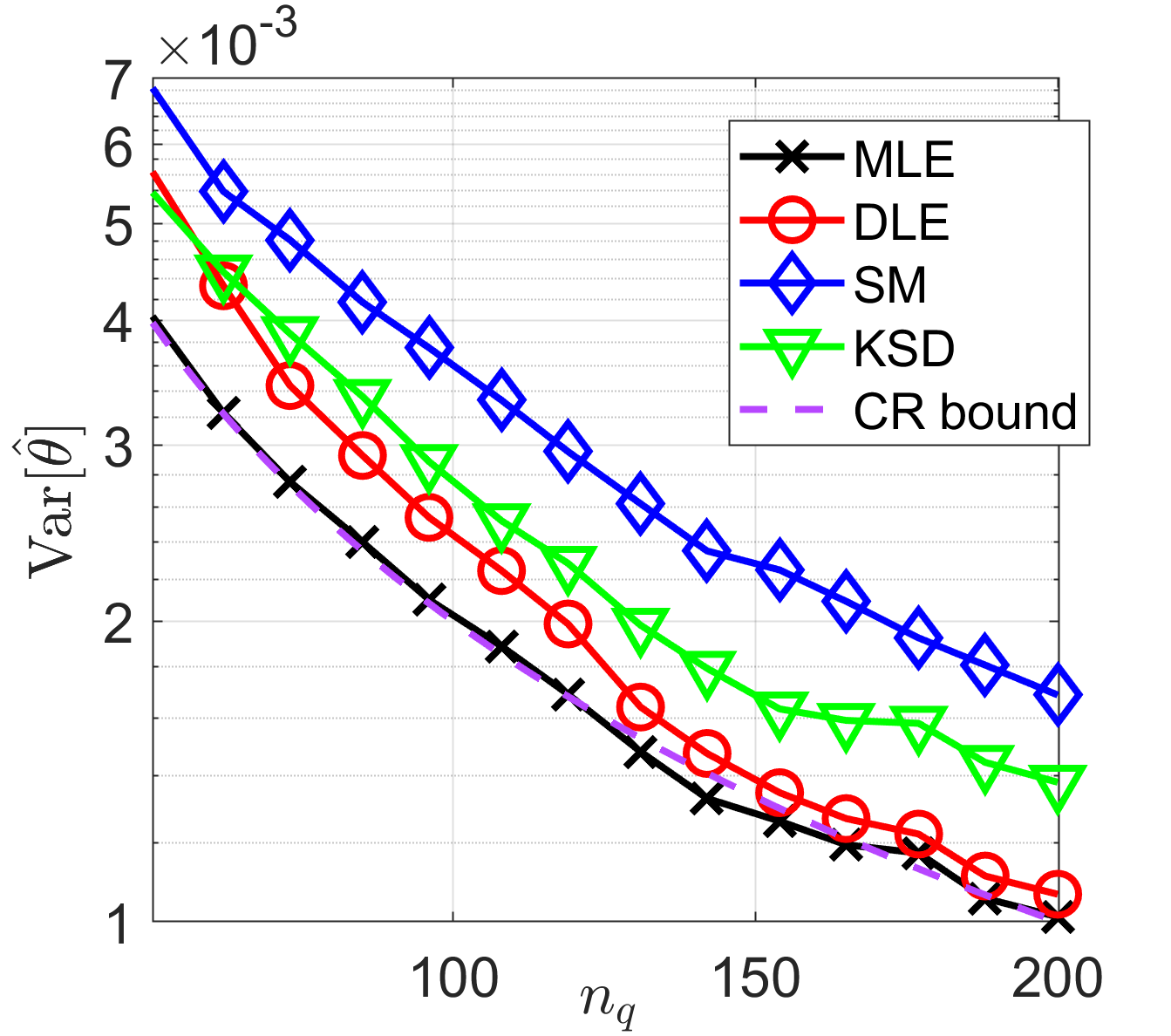}}
    \subfloat[\small $\mathrm{Var}(\hat{\boldtheta})$, Gaussian mix.]{
        \label{fig.gm.var}
        \includegraphics[width=.24\textwidth]{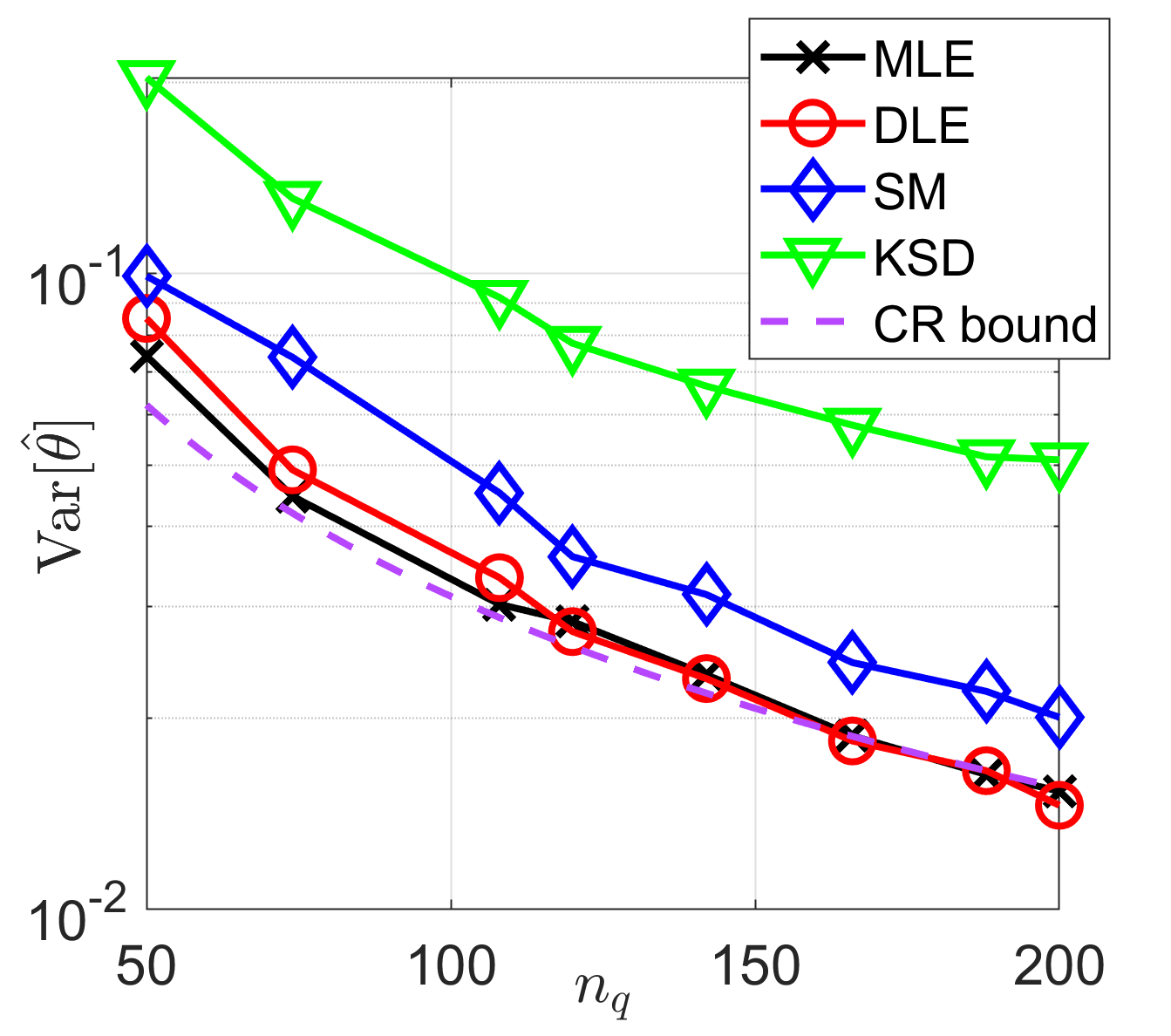}}
    \caption{Theoretical Prediction values vs. Empirical results}
    \label{fig.asym}
\end{figure}
\begin{figure}
    \centering
    \includegraphics[width=.99\textwidth]{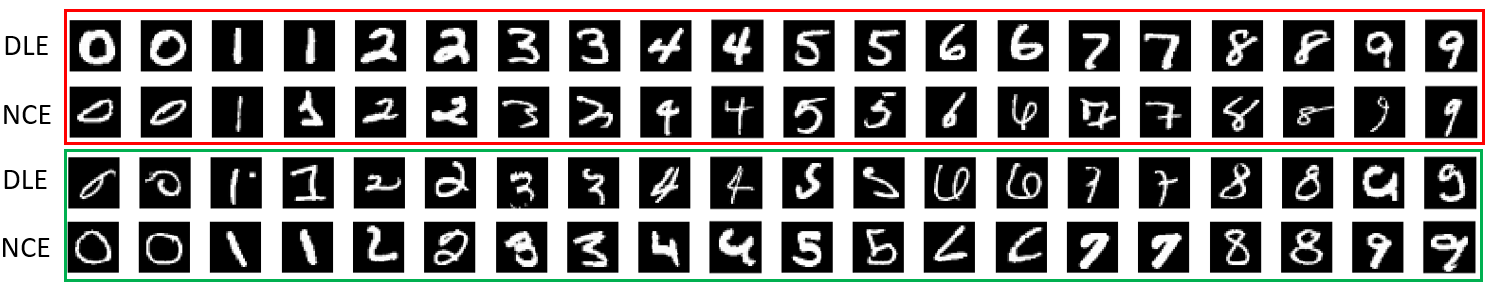}
    \caption{\small MNIST images with the highest (upper red box) and the lowest unnormalized density (lower green box) estimated on each digit by DLE and NCE.}
	\label{fig.mnist}
	\vspace*{-3mm}
\end{figure}
To examine the asymptotic distribution of $\sqrt{n_q}[\hat{\boldtheta}-\boldtheta^*]$, we design an intractable exponential family model $\bar{p}(\boldx;\boldtheta) := \exp\left[\boldeta(\boldtheta)^\top \boldpsi(\boldx)\right]$, where \[\boldpsi(\boldx) := [\sum_{i=1}^d x^2_i, x_1x_2,\sum_{i=3}^d x_1x_i, \tanh(\boldx) ]^\top, \boldeta(\boldtheta) := \left[-.5,.6,.2,0,0,0,\boldtheta \right]^\top, \boldx \in \mathbb{R}^5, \boldtheta \in \mathbb{R}^2.\] $\tanh(\boldx)$ is applied in an element-wise fashion. The feature function of the Stein feature is chosen as $\boldf(\boldx) := \tanh(\boldx) \in \mathbb{R}^5$. Due to the $\tanh$ function, $\bar{p}(\boldx;\boldtheta)$ does not have a closed form normalization term. 
We draw $n_q=500$ samples from $p(\boldx;\boldzero)$ as $X_q$. Given we set $ \boldtheta^* = \boldzero$, $X_q$ actually comes from a tractable distribution. However the intractability of $\bar{p}(\boldx;\boldtheta)$ does not allow us to perform MLE straight away. 

We run DLE 6000 times with new batch of $X_q$ each time and obtain an empirical distribution of $ \sqrt{n_q}\hat{\boldtheta}$. We show qqplots of its marginal distributions vs. $\mathcal{N}(0,V_{1,1})$, $\mathcal{N}(0,V_{2,2})$, the asymptotic distribution predicted by Theorem \ref{thm.normality2} whose variance $V$ is approximated by $X_q$ and $\hat{\boldtheta}$. \Cref{fig.asy.qqp} 
shows all quantiles between the empirical marginals and predicted marginals are very well aligned. We also scatter-plot $\sqrt{n_q}\hat{\boldtheta}$ together with the predicted 95\% and 99.9\% confidence interval in \Cref{fig.asy.theta}. It can be seen that the empirical joint distribution of $\sqrt{n_q}\hat{\boldtheta}$  has the same elongated shape as predicted by Theorem \ref{thm.normality2} and agrees with the predicted confidence intervals nicely. 

One of our major contributions is proving DLE attains the Cram{\'e}r-Rao bound. We now compare the variances of the estimated parameter $\hat{\theta}$ using Gamma $p(x;\theta) = \Gamma(5,\theta), \theta^* =1$ and Gaussian mixture model $p(x; \theta) = .5\mathcal{N}(\theta,1) + .5\mathcal{N}(1,1), \theta^* = -1$ across DLE, SM and KSD. $\mathrm{Var}_{n_q}[\hat{\theta}]$ are shown on Figure \ref{fig.gamma.var} and \ref{fig.gm.var}. For DLE, we set $\boldf(\boldx):=[x,x^2]$ and for KSD, we use a polynomial kernel with degree 2.
Note we particularly choose $p(x;\theta)$ to be tractable so we can compute MLE and Cram{\'e}r-Rao bound easily. It can be seen that all estimators have decreasing variances and MLE, being one of the minimum variance estimators, has the lowest variance. However, DLE has the second lowest variances in both cases and quickly converges to Cram{\'e}r-Rao bound after $n_q=150$. In comparison, both KSD and SM maintain higher levels of variances.  

\subsection{Unnormalized Model Using Pre-trained Deep Neural Network (DNN)}
In this experiment, we create an exponential family model $\bar{p}(\boldx;\boldtheta_\mathrm{i}) := \exp[\boldtheta_\mathrm{i}^\top \boldpsi(\boldx)], \boldx \in \mathbb{R}^{784}$ where $\boldpsi(\boldx) \in \mathbb{R}^{20}$ is a pre-trained 3-layer DNN. 
$\boldpsi(\boldx)$ is trained using a  logistic regression so that the classification error is minimized
on the full MNIST dataset over all digits. Clearly, $\bar{p}(\boldx;\boldtheta_\mathrm{i})$ does not have a tractable normalization term.  The dataset $X_{q_i}$  contains $n_q = 100$ randomly selected images from a \emph{single digit} $\mathrm{i}$ and we use DLE and NCE to estimate $\hat{\boldtheta}_\mathrm{i}$ for each digit $\mathrm{i}$. For DLE, we set $\boldf(\boldx) = \boldpsi(\boldx)$. Though we can only obtain an unnormalized density for each digit, it can be used to rank images and find potential inliers and outliers. 

In \Cref{fig.mnist} we show images that are ranked either among the top two or bottom two places when sorted by $\log\bar{p}(\boldx;\hat{\boldtheta}_\mathrm{i})$, for each digit $\mathrm{i}$. It can be seen that, when $\hat{\boldtheta}$ is estimated by DLE, images ranked the highest are indeed typical looking images, while the lowest ranking images tend to be outliers in that digit group. 
However, in comparison, when $\hat{\boldtheta}$ is estimated by NCE, some highest ranked images are distorted while some lowest ranked image look very regular.  
This experiment shows the usefulness of DLE as a model inference method when working with a complex model (DNN) on a high dimensional dataset ($d=784$) using relatively small number of samples ($n_q = 100$).
\vspace*{-7mm}

\section{Conclusion and Discussion}
\vspace*{-2mm}
In this paper, we introduce a model inference method for unnormalized statistical models. First, Stein density ratio estimation is used to fit a ratio and to approximate the KL divergence. The model inference is done by minimizing such an approximated KL divergence. Despite promising theoretical and experimental results, future works are needed to demonstrate a systematic way of choosing Stein features in different scenarios as the performance of DLE  depends heavily on such choices. 

\section*{Acknowledgements}
This work was supported by The Alan Turing Institute under the EPSRC grant EP/N510129/1.
TK was partially supported by JSPS KAKENHI Grant Number 15H01678, 15H03636, 16K00044, and 19H04071. Authors would like to thank Dr. Carl Henrik Ek and three anonymous reviewers for their insightful comments.

\bibliography{main}
\bibliographystyle{plainnat}
\newpage
\onecolumn
\appendix
\begin{center}
{\LARGE{}{}{}{}{}{}\ourtitle{}} 
\par\end{center}
	
\begin{center}
\textcolor{black}{\Large{}{}{}{}{}{}Supplementary}{\Large{}{}{}{}{}
} 
\par\end{center}
\section{Examples}
\subsection{Examples of Stein Features $T_{\boldtheta}\boldf(\boldx)$}
\label{sec.stein.ex}
\begin{example}
	\label{ex.fea}
	Let $p = \mathcal{N}(0,1)$, $T_{\mathcal{N}(0,1)} 1 = 0$, then $T_{\mathcal{N}(0,1)} x = -x$ and $T_{\mathcal{N}(0,1)} x^2/2 = -x^2+1$.
\end{example}
As we see, Stein features with respect to $\mathcal{N}(0,1)$ using monomials of $x$ are same-order polynomial terms of $x$ which have been widely used as function basis in various function fitting applications.
\subsection{Assumption \ref{assum.info.hessian} Examples}
\label{sec.exp.ass2}

\begin{example}
	When $f(x) = 0$, by the definition of Stein feature at Section \ref{sec.stein.fea}, $T_{\theta} f(x) \equiv 0$. Our density ratio model does \emph{not} have any discriminative power and become a constant function $1$. 
	We can see
	$
	H_{\delta, \delta} =0,
	H_{\delta,\theta} = 0
	$ regardless what $\delta$ and $\theta$ are chosen.
	Thus, Assumption \ref{assum.info.hessian} is not satisfied here. See \eqref{eq.deri.1} and \eqref{eq.deri.2} in Section \ref{sec.hessian.formula} in Appendix for the exact formulas of $H_{\delta, \delta}$ and $H_{\delta, \theta}$. 
\end{example}

\begin{example}
	When $f(x) := x$ and $p(x;\theta) := \mathcal{N}(\theta,1)$, our density ratio model becomes a linear discriminative function (See Example \ref{ex.fea}).
	From \eqref{eq.deri.1} and \eqref{eq.deri.2}  we can see, when $\theta = \theta^*$ and $\delta = 0$, 
	$
	H^*_{\delta, \delta} = -\frac{1}{n_q} \sum_{i=1}^{n_q} (x_q^{(i)}-\theta^*)^2 
	$
	which is essentially the negative sample variance and $H^*_{\delta, \theta} =   \frac{1}{n_q} \sum_{i=1}^{n_q} \nabla_\theta (x_q^{(i)}-\theta) = -1$. Given $n_q$ is sufficiently large, $\Lambda_\mathrm{min}$ and $\Lambda'_\mathrm{min}$ is reasonably small and $\Lambda_\mathrm{max}$ is reasonably large, Assumption \ref{assum.info.hessian} should hold at the optimal point $(\theta^*, 0)$ with high probability. We omit the analysis when $\delta$ and $\theta$ are slightly deviated from their optimal values due to the page limit. Nonetheless, it can be analysed with some extra regularity conditions.
\end{example}

\subsection{Example of Asymptotic Efficient Choice of $\boldf(\boldx)$}
\label{sec.exp.asym}

\begin{example}
	Consider the univariate Gaussian distribution 
	$p(x;\boldtheta)=\exp\big\{\theta_1x + \theta_2x^2\big\}/Z(\boldtheta)$ for
	$x\in\mathbb{R},\boldtheta=(\theta_1,\theta_2)$, 
	where $\theta_1\in\mathbb{R}, \theta_2<0$, and $Z(\boldtheta)$ is the normalization constant.
	The score function is 
	$
	s_1(x;\boldtheta) =x-\frac{1}{Z(\boldtheta)}\partial_{\theta_1}Z(\boldtheta),
	s_2(x;\boldtheta) =x^2-\frac{1}{Z(\boldtheta)}\partial_{\theta_2} Z(\boldtheta). 
	$
	Let us consider the Stein feature vector for $\boldf(x)=(x,x^2/2)^{\top}$,
	$
	T_{\boldtheta}\boldf(\boldx) 
	=(\theta_1  + 2\theta_2 x,\, 1+\theta_1x  + 2\theta_2 x^2)^{\top}.
	$
	We know that 
	$\theta_1Z(\boldtheta)+2\theta_2\partial_{\theta_1}Z(\boldtheta)=0$ and
	$Z(\boldtheta)+\theta_1\partial_{\theta_1}Z(\boldtheta)+2\theta_2\partial_{\theta_2}Z(\boldtheta)=0$ (see \citep{hayakawa2016} for details). Thus, 
	$\begin{pmatrix}T_{\boldtheta}f_1(\boldx)\\ T_{\boldtheta}f_2(\boldx)\end{pmatrix}
	=
	\begin{pmatrix}2\theta_2 & 0\\ \theta_1 & 2\theta_2   \end{pmatrix}
	\begin{pmatrix}s_1\\ s_2 \end{pmatrix}.$
	The coefficient matrix is invertible as long as $\theta_2 \neq 0$.
	Hence, the DLE with the above $\boldf$ achieves the asymptotic efficiency bound. 
\end{example}	

 \section{Proofs}
For simplicity, we write all $\sum_{i=1}^{n_q} g(\boldx_q^{(i)})$ as $\sum_{i=1}^{n_q} g(\boldx^{(i)})$ from now on as samples always come from dataset $X_q$. 
See Table \ref{tbl.notation} for all defined notations. 

\begin{table}[t]
	\caption{Notations of Symbols }
	\label{tbl.notation}
	\vskip 0.15in
	\begin{center}
		\begin{small}
				\begin{tabular}{ c | >{\centering\arraybackslash} p{10.2cm}  }
					\hline
					Symbol & Definition  \\ \hline
					$\ell(\bolddelta, \boldtheta)$ & $\frac{1}{n_q}\sum_{i=1}^{n_q}\log r_\boldtheta(\boldx_{q}^{(i)};\bolddelta)$, log likelihood ratio \\
					$\nabla \ell(\bolddelta_0, \boldtheta_0)$ & $\nabla_{(\bolddelta, \boldtheta)} \ell(\bolddelta_0, \boldtheta_0)\vert_{\bolddelta=\bolddelta_0, \boldtheta=\boldtheta_0}$ \\
					$\nabla_\bolddelta \ell(\bolddelta_0, \boldtheta_0)$ & $\nabla_\bolddelta \ell(\bolddelta_0, \boldtheta_0)$  \\ \hline
					$\boldH$ & $\nabla_{(\bolddelta, \boldtheta)}^2\ell(\bolddelta,\boldtheta)$, Hessian of likelihood\\
					$\boldH_{\bolddelta, \boldtheta}$ & $\nabla_\bolddelta\nabla_\boldtheta \ell(\bolddelta, \boldtheta)$, submatrix of Hessian. \\   
					\hline
					$\mathrm{Ball}(R, \boldx_0)$ & $\ell_2$ ball with radius $R$ centered at $\boldx_0$\\\hline
					$\|A\|$ & $\ell_2$ norm of a vector $A$ or the \textbf{spectral norm} of a matrix $A$\\\hline
					$\bolds(\boldx;\boldtheta)\in \mathbb{R}^{\mathrm{dim}(\boldtheta)}$ & $\nabla_\boldtheta \log p(\boldx,\boldtheta)$, Score function of $p_\boldtheta$
					\\
					$\bolds$& $s(\boldx;\boldtheta^*)$\\
					\hline
				\end{tabular}
		\end{small}
	\end{center}
	\vskip -0.1in
\end{table}

\subsection{Proof of Lemma \ref{prop.stein.equality}}
\label{sec.proof.stein.eq}
\begin{proof}
	Our proof below is similar to the proof of Lemma 4 in \citep{Hyvaerinen2005}.
	It can be seen that 
	\begin{align*}
	\mathbbE_{p_\boldtheta} [T_{\boldtheta} f_i(\boldx)] &= \int p(\boldx;\boldtheta)\left[\langle\nabla_\boldx \log p(\boldx;\boldtheta),  \nabla_{\boldx} f_i(\boldx)\rangle + \mathrm{trace}(\nabla^2_\boldx f_i(\boldx))\right] \dx\\
	&= \int \langle \nabla_\boldx p(\boldx;\boldtheta), \nabla_\boldx f_i(\boldx)\rangle + p(\boldx; \boldtheta)\cdot\mathrm{trace}(\nabla^2_\boldx f_i(\boldx))\dx.
	\end{align*}
	Let us rewrite  $\mathbbE_{p_\boldtheta} [T_{\boldtheta} f_i(\boldx)]$ as nested integrals over each component of $\boldx$: 
	\begin{align}
	&\mathbbE_{p_\boldtheta} [T_{\boldtheta} f_i(\boldx)] \notag\\
	\label{eq.stein.iden.1}
	=& \sum_{j=1}^{d}\int_{\boldx_{\setminus j}} \int_{x_j} \partial_{x_j} f_i(\boldx) \cdot \partial_{x_j} p(\boldx;\boldtheta)  + p(\boldx;\boldtheta)\cdot\partial^2_{x_j} f_i(\boldx) \mathrm{d} x_j \dx_{\setminus j},\\
	\label{eq.stein.iden.2}
	=& \sum_{j=1}^{d} \int_{\boldx_{\setminus j}}
	\underbrace{\left[ p(\boldx;\boldtheta)\partial_{x_j}f_i(\boldx)
		\right]_{x_j \to -\infty}^{{x_j \to +\infty}}}_{0, \text{by assumption}}  \dx_{\setminus j}
	- \int_{\boldx_{\setminus j}}\int_{x_j} p(\boldx;\boldtheta)\left[
	\partial^2_{x_j} f_i(\boldx) - 
	\partial^2_{x_j} f_i(\boldx) \right] \mathrm{d} x_j \dx_{\setminus j} \\
	=& 0.
	\end{align}
	where $\boldx_{\setminus j}$ contains all the components in $\boldx$ except the $j$-th component. 
	The equality from \eqref{eq.stein.iden.1} to \eqref{eq.stein.iden.2} is due to one dimensional integration by parts formula. The first term in \eqref{eq.stein.iden.2} is zero as the product of $p(\boldx)$ and $\partial_{x_j}f_i(\boldx)$ is asssumed to be zero when $x_j$ takes the limit to $+/-\infty$.
	Our assumption holds for all $i,j$, so we can assert $\forall_{i} \mathbbE_{p_\boldtheta} [T_{\boldtheta} f_i(\boldx)] = 0$ and $\mathbbE_{p_\boldtheta} [T_{\boldtheta} \boldf(\boldx)]=\boldzero$ by its construction.
\end{proof}

\subsection{Derivations of $\nabla^2_\delta\ell(\delta, \boldtheta)$ and $\nabla_{\delta,\boldtheta}\ell(\delta, \boldtheta)$ with $f(\boldx): \mathbb{R}^{d} \to \mathbb{R}$}
\label{sec.hessian.formula}
\begin{align}
\nabla^2_\delta\ell(\delta, \boldtheta) &= -\frac{1}{n_q} \sum_{i=1}^{n_q} \frac{\left[T_{\boldtheta} f(\boldx^{(i)})\right]^2}{r^2_\boldtheta(\boldx^{(i)};\delta)} + 0,\label{eq.deri.1}\\
\nabla_{\delta,\boldtheta}\ell(\delta, \boldtheta) &= -\frac{1}{n_q} \sum_{i=1}^{n_q} \frac{T_{\boldtheta} f(\boldx^{(i)})}{r^2_\boldtheta(\boldx^{(i)};\delta)} \nabla_\boldtheta r_\boldtheta(\boldx^{(i)};\delta) + \frac{1}{n_q} \sum_{i=1}^{n_q} \frac{1}{r_\boldtheta(\boldx^{(i)};\delta)} \nabla_\boldtheta T_{\boldtheta} f(\boldx^{(i)}).\label{eq.deri.2}
\end{align}

\subsection{Proof of Proposition \ref{prop.ass2.expo}}
\label{sec.proof.ass2.expo}
\begin{proof}
	First, 
	the definition of $\Delta_{n_q}$ gives the boundedness of our ratio, i.e., $\frac{1}{C_\mathrm{ratio}}\le r_\boldtheta(\boldx;\bolddelta)\le C_\mathrm{ratio}, \forall \boldx \in X_q, \forall \boldtheta \in \Theta$.
	
	Second, $-\boldH_{\bolddelta,\bolddelta} = \frac{1}{n_q}\sum_{i=1}^{n_q} \frac{1}{r^2_\boldtheta\left(\boldx^{(i)};\bolddelta\right)} \cdot 
	{T\boldpsi^{(i)}} 
	{T\boldpsi^{(i)}}^\top$, where $T\boldpsi^{(i)}$ is an abbreviation of $T_{\boldtheta}\boldpsi(\boldx^{(i)})$. It is a sum over ratio weighted positive semi-definite matrices so we can lower bound its minimum eigenvalue using the lower bound of the ratio: 
	\[
	\lambda_\mathrm{min}(-\boldH_{\bolddelta, \bolddelta}) \ge \frac{1}{C^2_\mathrm{ratio}} \lambda_{\mathrm{min}} \left( \frac{1}{n_q} \sum_{i=1}^{n_q} 
	T\boldpsi^{(i)} {T\boldpsi^{(i)}}^\top\right) > \frac{\Lambda''_{\text{min}}}{C^2_\mathrm{ratio}} > 0, 
	\text{with high prob.,}
	\] due to our assumption. Similarly, we can also upper-bound its maximum eigenvalue
	\[
		\lambda_\mathrm{max}(-\boldH_{\bolddelta, \bolddelta}) \le C^2_\mathrm{ratio} \lambda_{\mathrm{max}} \left( \frac{1}{n_q} \sum_{i=1}^{n_q} 
		T\boldpsi^{(i)} {T\boldpsi^{(i)}}^\top\right) \le C^2_\mathrm{ratio} \Lambda''_{\text{max}}, 
		\text{with high prob.,}
	\]
	
	Third, $-\boldH_{\boldtheta,\boldtheta} = \frac{1}{n_q}\sum_{i=1}^{n_q} \frac{1}{r^2_\boldtheta\left(\boldx^{(i)};\bolddelta\right)} \boldJ_\boldx \boldpsi(\boldx^{(i)}) \boldJ_\boldx\boldpsi(\boldx^{(i)})^\top \bolddelta \bolddelta^\top \boldJ_\boldx \boldpsi(\boldx^{(i)})^\top \boldJ_\boldx\boldpsi(\boldx^{(i)})$. We can see \[\|\boldH_{\boldtheta,\boldtheta}\|
	\le
	\frac{C^2_\mathrm{ratio}\cdot \|\bolddelta\|^2}{n_q} \sum_{i=1}^{n_q} \|\boldJ_\boldx \boldpsi(\boldx^{(i)})\|^4  \le C^2_{\mathrm{ratio}}C_2 \cdot \|\bolddelta\|^2\le \frac{C^2_{\mathrm{ratio}}C_2 T}{\sigma(n_q)^2}.\]
	
	Fourth, using the fact that $-\boldH_{\bolddelta,\bolddelta}$ is a positive definite matrix, which we have just proved, we can see
	\begin{align*}
	\lambda_\mathrm{min}\left\{-\boldH_{\boldtheta, \bolddelta}\boldH_{\bolddelta,\bolddelta}^{-1}\boldH_{\bolddelta,\boldtheta}\right\} &= \lambda_\mathrm{min}( -\boldH_{\bolddelta,\bolddelta}^{-1}
	\boldH_{\bolddelta, \boldtheta}\boldH_{\boldtheta, \bolddelta}) \\
	&\ge \lambda_\mathrm{min}( -\boldH_{\bolddelta,\bolddelta}^{-1})
	\lambda_{\mathrm{min}}(	\boldH_{\bolddelta, \boldtheta}\boldH_{\boldtheta, \bolddelta})\\
    &=
	\frac{\lambda_\mathrm{min}(		\boldH_{\bolddelta, \boldtheta}\boldH_{\boldtheta, \bolddelta})}{\lambda_\mathrm{max}(-\boldH_{\bolddelta,\bolddelta})}
	\ge \frac{\lambda_\mathrm{min}(		\boldH_{\bolddelta, \boldtheta}\boldH_{\boldtheta, \bolddelta})}{C^2_\mathrm{ratio}\Lambda''_\mathrm{max}},
	\end{align*} 
	where 2nd line is due to Theorem 7, \citep{Merikoski2004}.
	So we only need to find a lower bound for $\lambda_\mathrm{min}(	\boldH_{\bolddelta, \boldtheta}\boldH_{\boldtheta, \bolddelta})$. We can write $\boldH_{\boldtheta, \bolddelta}$ as 
	\begin{align}
	\label{eq.Hthetadelta}
	\boldH_{\boldtheta, \bolddelta} = &\underbrace{\frac{1}{n_q} \sum_{i=1}^{n_q}\frac{1}{r_\boldtheta(\boldx^{(i)};\bolddelta)} \boldJ_\boldx \boldpsi(\boldx^{(i)}) \boldJ_\boldx \boldpsi(\boldx^{(i)})^\top}_{\boldA} \\
	-&\underbrace{\frac{1}{n_q} \sum_{i=1}^{n_q}\frac{1}{r^2_\boldtheta(\boldx^{(i)};\bolddelta)} \boldJ_\boldx \boldpsi(\boldx^{(i)}) \boldJ_\boldx \boldpsi(\boldx^{(i)})^\top 
	\bolddelta T_{\boldtheta} \boldpsi(\boldx^{(i)}) ^\top }_{\boldB}
	\end{align}
	Therefore $\boldH_{\bolddelta, \boldtheta}\boldH_{\boldtheta, \bolddelta}$ can be written as 
	\begin{align*}
		\boldA\boldA^\top - \boldA\boldB^\top - \boldB\boldA^\top + \boldB\boldB^\top.
	\end{align*}
	Since we are analyzing the minimum eigenvalue, we can safely ignore the last term $\boldB\boldB^\top$ as it is positive semi-definite. This gives the following inequality:
	\begin{align*}
		\lambda_{\mathrm{min}}\left\{\boldA\boldA^\top - \boldA\boldB^\top - \boldB\boldA^\top\right\} &\ge \lambda_{\mathrm{min}}\left\{\boldA\boldA^\top\right\} + \lambda_{\mathrm{min}}\left\{- \boldA\boldB^\top - \boldB\boldA^\top\right\}\\
		&\ge 
		\lambda_{\mathrm{min}}\left\{\boldA\boldA^\top\right\} - \|\boldA\boldB^\top + \boldB\boldA^\top\|
	\end{align*}
	As $\boldA$ is a sum of ratio weighted positive semi-definite matrices, we can use the same trick in the second step to lower bound its eigenvalue using the lower bound of the density ratio, eventually, using our assumption on $\lambda_{\mathrm{min}}\{\frac{1}{n_q}\sum_{i=1}^{n_q}\boldJ^{(i)}{\boldJ^{(i)}}^\top\}\ge C_3$, we can get,
	\begin{align*}
		\lambda_{\mathrm{min}}(\boldA) \ge \frac{C_3}{C_\mathrm{ratio}},
		\lambda_{\mathrm{min}}(\boldA\boldA^\top) \ge \lambda_{\mathrm{min}}(\boldA) \cdot \lambda_{\mathrm{min}}(\boldA) \ge  \frac{C_3^2}{C^2_\mathrm{ratio}}. 
	\end{align*}
	
	Now we analyze $\|\boldA\boldB^\top + \boldB\boldA^\top\|$ which is further upperbounded by $2\|\boldA\|\|\boldB\|$.
	
	Similarly to how $\lambda_{\mathrm{min}}(\boldA)$ is bounded, we can upper-bound $\|\boldA\|$  using the upperbound of the ratio:  $\|\boldA\| \le C_\mathrm{ratio}C_4$.
	Let us write 
	$\boldB = \frac{1}{n_q}\sum_{i=1}^{n_q} \frac{1}{r^2_i} \boldJ^{(i)} {\boldJ^{(i)}}^\top \bolddelta {T \boldpsi^{(i)}}^\top$
	where $\boldJ^{(i)}$ and $r_i$ are abbreviations of $\boldJ_\boldx \boldpsi(\boldx^{(i)})$ and $r_\boldtheta(\boldx^{(i)};\bolddelta)$. It can be seen that
	\begin{align*}
		\|\boldB\| \le \frac{1}{n_q}\sum_{i=1}^{n_q} \|\frac{1}{r^2_i} \boldJ^{(i)} {\boldJ^{(i)}}^\top \|\cdot \|\bolddelta T \boldpsi^{(i)}\| &\le	\frac{1}{n_q}\sum_{i=1}^{n_q} \|\frac{1}{r^2_i} \boldJ^{(i)} {\boldJ^{(i)}}^\top \|\cdot \|\bolddelta \|\cdot \|T \boldpsi^{(i)}\|,\\
		&\le C^2_\mathrm{ratio} \cdot C_5 T/\sigma(n_q).
	\end{align*}
	Now we can bound
	\begin{align*}
		\lambda_{\mathrm{min}}\left\{\boldA\boldA^\top - \boldA\boldB^\top - \boldB\boldA^\top + \boldB\boldB^\top\right\} & \ge 
		\lambda_{\mathrm{min}}\left\{\boldA\boldA^\top\right\} - 2\|\boldA\|\|\boldB\| \\ 
		&\ge \frac{C^2_3}{C^2_\mathrm{ratio}} - C^3_\mathrm{ratio} C_4 \cdot  C_5 T/\sigma(n_q)
	\end{align*}
	
	There exists a constant $N>0$, such that when $n_q>N$, 
	\begin{align*}
		\lambda_\mathrm{min}\left\{-\boldH_{\boldtheta, \bolddelta}\boldH_{\bolddelta,\bolddelta}^{-1}\boldH_{\bolddelta,\boldtheta}\right\}
		\ge \frac{\lambda_{\mathrm{min}} (\boldH_{\bolddelta, \boldtheta}\boldH_{\boldtheta, \bolddelta})}{C^2_\mathrm{ratio}\Lambda''_\mathrm{max}} 
		\ge& \frac{C^2_3}{C^4_\mathrm{ratio}\Lambda'_\mathrm{max}}  -  \frac{C_\mathrm{ratio}C_4 \cdot  C_5 T}{\sigma(n_q)\Lambda''_\mathrm{max}} \\
		\ge& \frac{C^2_\mathrm{ratio} C_2 T}{\sigma(n_q)^2}\ge \|\boldH_{\boldtheta,\boldtheta} \|.
	\end{align*}
	
	Finally we analyze $\left\|\boldH_{\boldtheta,\bolddelta}\boldH_{\bolddelta,\bolddelta}^{-1}\right\|$. $\left\|\boldH_{\boldtheta,\bolddelta}\boldH_{\bolddelta,\bolddelta}^{-1}\right\| \le \left\|\boldH_{\boldtheta,\bolddelta}\right\|\cdot\left\|\boldH_{\bolddelta,\bolddelta}^{-1}\right\|$. As $-\boldH_{\bolddelta,\bolddelta}$ is positive definite, the operator norm of its inverse is the inverse of its minimum eigenvalue, which is upperbounded by $C^2_{\mathrm{ratio}}/\Lambda''_\mathrm{min}$. On the other hand, we can rewrite \eqref{eq.Hthetadelta} as 
	$\boldH_{\boldtheta, \bolddelta}  = \frac{1}{n_q} \sum_{i=1}^{n_q} \frac{1}{r_i} \cdot  \boldJ^{(i)}{\boldJ^{(i)}} ^\top \cdot  \left( \mathrm{Iden} - \frac{1}{r_i}\cdot \bolddelta T {\boldpsi^{(i)}}^\top\right)$, so
	\begin{align*}
	\left\| \boldH_{\boldtheta, \bolddelta} \right\| &\le \frac{1}{n_q} \sum_{i=1}^{n_q} \frac{1}{r_i} \cdot \left\|  \boldJ^{(i)}{\boldJ^{(i)}} ^\top \cdot  \left( \mathrm{Iden} - \frac{1}{r_i}\cdot \bolddelta T {\boldpsi^{(i)}}^\top\right)\right\|\\
	&\le \frac{C_\mathrm{ratio}}{n_q} \sum_{i=1}^{n_q} \|  \boldJ^{(i)}{\boldJ^{(i)}}^\top \| \cdot
	\underbrace{\|\mathrm{Iden} - \frac{1}{r_i}\cdot \bolddelta T {\boldpsi^{(i)}}^\top\|}_{\boldC}
	\end{align*}
	From calculation, we know $\|\boldC\| \le 1 + |(r_i-1)/r_i|\le 2+C_\mathrm{ratio}$. Therefore 
	$\left\| \boldH_{\boldtheta, \bolddelta} \right\| \le  \frac{C^2_\mathrm{ratio} + 2C_\mathrm{ratio}}{n_q} \sum_{i=1}^{n_q}   \|\boldJ^{(i)}{\boldJ^{(i)}} ^\top\| \le (C^2_\mathrm{ratio} + 2C_\mathrm{ratio})C_4$. Therefore $\left\|\boldH_{\boldtheta,\bolddelta}\boldH_{\bolddelta,\bolddelta}^{-1}\right\|$ is upperbounded by $(C^4_\mathrm{ratio} + 2C^3_\mathrm{ratio})C_4/\Lambda''_\mathrm{min}$.
	
	Refer to \citep{Merikoski2004,Jiu2007} for inequalities of eigenvalue of matrix summation and product.
\end{proof}

\subsection{Proof of Theorem \ref{thm.main}}
\label{sec.main.proof}
\begin{proof}
	We denote Hessian $\boldH$ as a block matrix: 
	\[
	\boldH = \nabla^2\ell(\bolddelta, \boldtheta) = \begin{pmatrix}
	\boldH_{11} & \boldH_{12}\\\boldH_{21} & \boldH_{22}
	\end{pmatrix} = \begin{pmatrix}
	\nabla^2_\bolddelta\ell(\bolddelta, \boldtheta) & \nabla_\bolddelta\nabla_\boldtheta\ell(\bolddelta, \boldtheta)\\\nabla_\boldtheta\nabla_\bolddelta\ell(\bolddelta, \boldtheta) & \nabla^2_\boldtheta\ell(\bolddelta, \boldtheta)
	\end{pmatrix},
	\]
	then Assumption \ref{assum.info.hessian} states that for every $\bolddelta \in \Delta_{n_q}$ and $\boldtheta \in \Theta$, $\lambda(\boldH_{21}\boldH_{11}^{-1}\boldH_{12})$ is lower bounded by $2\left\|\boldH_{22}\right\|$ and $\left\|\boldH_{21}\boldH_{11}^{-1}\right\|$ is upper bounded. 
	
	We can write the optimality condition of \eqref{eq.thm.obj} and expand them  at $(\bolddelta^*\equiv\boldzero, \boldtheta^*)$:
	\begin{align}
	\nabla_\bolddelta \ell(\hat{\bolddelta},\hat{\boldtheta}) &= \boldzero = \nabla_\bolddelta \ell(\bolddelta^*,\boldtheta^*) + \bar{\boldH}_{11} (\hat{\bolddelta} - \bolddelta^*)
	+ \bar{\boldH}_{12} (\hat{\boldtheta} - \boldtheta^*)\label{eq.opt.1}\\
	\nabla_\boldtheta \ell(\hat{\bolddelta},\hat{\boldtheta}) &= \boldzero = \nabla_\boldtheta \ell(\bolddelta^*,\boldtheta^*) + \bar{\boldH}_{21} (\hat{\bolddelta} - \bolddelta^*)
	+ \bar{\boldH}_{22} (\hat{\boldtheta} - \boldtheta^*)\label{eq.opt.2},
	\end{align}
	where $\bar{\boldH}$ is the Hessian evaluated at a  $(\bar{\bolddelta},\bar{\boldtheta})$ which is in between $(\hat{\bolddelta},\hat{\boldtheta})$ and $(\bolddelta^*, \boldtheta^*)$ in \emph{an element-wise fashion}. This expansion is basically one-dimensional mean-value theorem applied on \emph{each individual dimension} of $\nabla_\bolddelta \ell(\hat{\bolddelta},\hat{\boldtheta})$ and $\nabla_\boldtheta \ell(\hat{\bolddelta},\hat{\boldtheta})$.
	
	Given \eqref{eq.opt.1} and \eqref{eq.opt.2} we can solve  equations for $\hat{\bolddelta} - \bolddelta^*$ and $\hat{\boldtheta} - \boldtheta^*$.
	
	From \eqref{eq.opt.1} we can get 
	\begin{align}
	\hat{\bolddelta} - \bolddelta^* = \bar{\boldH}_{11}^{-1}\left[-\nabla_\bolddelta \ell(\bolddelta^*,\boldtheta^*)
	- \bar{\boldH}_{12} (\hat{\boldtheta} - \boldtheta^*)\right]\label{eq.opt.3}.
	\end{align}
	Substituting \eqref{eq.opt.3} into \eqref{eq.opt.2} we get 
	\begin{align*}
	\boldzero = \nabla_\boldtheta \ell(\bolddelta^*,\boldtheta^*) - \bar{\boldH}_{21}\bar{\boldH}_{11}^{-1} \nabla_\bolddelta \ell(\bolddelta^*,\boldtheta^*)  +\left[-\bar{\boldH}_{21}\bar{\boldH}_{11}^{-1}\bar{\boldH}_{12} + \bar{\boldH}_{22}\right]\left(\hat{\boldtheta} - \boldtheta^* \right).
	\end{align*}
	Rearranging terms, we get 
	\begin{align}
	\hat{\boldtheta} - \boldtheta^* =& \left[\bar{\boldH}_{21}\bar{\boldH}_{11}^{-1}\bar{\boldH}_{12} - \bar{\boldH}_{22}\right]^{-1} \left(\nabla_\boldtheta \ell(\bolddelta^*,\boldtheta^*) - \bar{\boldH}_{21}\bar{\boldH}_{11}^{-1} \nabla_\bolddelta \ell(\bolddelta^*,\boldtheta^*) \right)\label{eq.thetadiff.0}\\
	=&\left[-\bar{\boldH}_{21}\bar{\boldH}_{11}^{-1}\bar{\boldH}_{12} + \bar{\boldH}_{22}\right]^{-1} \bar{\boldH}_{21}\bar{\boldH}_{11}^{-1} \nabla_\bolddelta \ell(\bolddelta^*,\boldtheta^*).\label{eq.thetadiff.rep}
	\end{align}
	The last line uses the fact that $\nabla_\boldtheta \ell(\bolddelta^*,\boldtheta^*) \equiv \boldzero$.
	
	Weyl's inequality states:
	\begin{align*}
	\lambda_\mathrm{min}(A+B) \ge \lambda_\mathrm{min}(A) + \lambda_\mathrm{min}(B).
	\end{align*}
	As $\bar{\bolddelta} \in \Delta_{n_q}$ and $\bar{\boldtheta} \in \Theta$,  $\bar{\boldH}$ is regulated by Assumption \ref{assum.info.hessian}. 
	Since \[\lambda_\mathrm{min}(-\bar{\boldH}_{21}\bar{\boldH}_{11}^{-1}\bar{\boldH}_{12}) \ge \Lambda_\mathrm{min}\] and \[\lambda_\mathrm{min}(\bar{\boldH}_{22}) \ge - \|\bar{\boldH}_{22}\| \ge -\frac{\Lambda_\mathrm{min}}{2}\] which are assumed by Assumption \ref{assum.info.hessian}, we have \[\lambda_\mathrm{min}(-\bar{\boldH}_{21}\bar{\boldH}_{11}^{-1}\bar{\boldH}_{12} + \bar{\boldH}_{22})\ge\Lambda_\mathrm{min}/2>0.\]
	
	Denote $-\bar{\boldH}_{21}\bar{\boldH}_{11}^{-1}\bar{\boldH}_{12} + \bar{\boldH}_{22}$ as $\bar{\boldH}/\bar{\boldH}_{22}$ (it is actually the Schur Complement of $\bar{\boldH}$). 
	Using Holder's inequality, we get 
	\begin{align}
	\|\hat{\boldtheta} - \boldtheta^*\| \le&  \left\|\left[\bar{\boldH}/\bar{\boldH}_{22}\right]^{-1}\right\| \left\|\bar{\boldH}_{21}\bar{\boldH}_{11}^{-1}\right\| \left\|\nabla_\bolddelta \ell(\bolddelta^*,\boldtheta^*) \right\| \notag\\
	\le& \frac{\left\|\bar{\boldH}_{21}\bar{\boldH}_{11}^{-1}\right\|}{\lambda_\mathrm{min}\left[\bar{\boldH}/\bar{\boldH}_{22}\right]} \cdot \left\|\nabla_\bolddelta \ell(\bolddelta^*,\boldtheta^*) \right\|
	\le \frac{2\Lambda_\mathrm{max}}{\Lambda_\mathrm{min}} \cdot \left\|\nabla_\bolddelta \ell(\bolddelta^*,\boldtheta^*) \right\|. \label{eq.consis.final}
	\end{align}
	
	Further, we have $\mathbbE_{q}\left[T_{\boldtheta^*} \boldf(\boldx)\right]= \mathbbE_{p_{\boldtheta^*}}\left[T_{\boldtheta^*} \boldf(\boldx)\right] =\boldzero$. The first equality is due to Assumption \ref{eq.model} and the second equality is given by Stein identity. 
	
	Therefore, $\nabla_\bolddelta \ell(\bolddelta^*,\boldtheta^*)  = \frac{1}{n_q}\sum_{i=1}^{n_q} T_{\boldtheta^*} f(\boldx^{(i)}) - \boldzero = \frac{1}{n_q}\sum_{i=1}^{n_q} T_{\boldtheta^*} f(\boldx^{(i)}) - \mathbbE_{q}\left[T_{\boldtheta^*} \boldf(\boldx)\right] $, which converges to $0$ in $\ell_2$ norm in probability due to Assumption \ref{ass.concentration}. This gives the convergence in probability of $\|\hat{\boldtheta} - \boldtheta^*\|$.
    Finite sample convergence rate can be given if the convergence rate of $\left\|\nabla_\bolddelta \ell(\bolddelta^*,\boldtheta^*)\right\|$ is known.

	Now we show the consistency of $\hat{\bolddelta}$. From \eqref{eq.opt.3} we can see that 
	\begin{align*}
	\hat{\bolddelta} - \bolddelta^* = -\bar{\boldH}_{11}^{-1}\nabla_\bolddelta \ell(\bolddelta^*,\boldtheta^*)
	- \bar{\boldH}_{11}^{-1}\bar{\boldH}_{12} (\hat{\boldtheta} - \boldtheta^*),
	\end{align*}
	and due to Holder's inequality, we get 
	\begin{align}
	\left\|\hat{\bolddelta} - \bolddelta^* \right\| &= \left\|-\bar{\boldH}_{11}^{-1}\right\|\left\|\nabla_\bolddelta \ell(\bolddelta^*,\boldtheta^*)\right\|
	+ \left\|\bar{\boldH}_{11}^{-1}\bar{\boldH}_{12}\right\| \left\|\hat{\boldtheta} - \boldtheta^*\right\|\notag\\
	&\le \frac{1}{\Lambda'_\mathrm{min}}\left\|\nabla_\bolddelta \ell(\bolddelta^*,\boldtheta^*)\right\| + \Lambda_\mathrm{max} \left\|\hat{\boldtheta} - \boldtheta^*\right\|. \label{eq.delta.diff.2}
	\end{align}
	Combine \eqref{eq.delta.diff.2} with \eqref{eq.consis.final} we get
	\begin{align*}
	\left\|\hat{\bolddelta} - \bolddelta^* \right\|\le   \frac{2\Lambda^2_\mathrm{max}\Lambda'_\mathrm{min}+\Lambda_\mathrm{min}}{\Lambda_\mathrm{min}\Lambda'_\mathrm{min}} \cdot \left\|\nabla_\bolddelta \ell(\bolddelta^*,\boldtheta^*) \right\|
	\end{align*}
	
    Again, due to Assumption \ref{ass.concentration}, $\left\|\nabla_\bolddelta \ell(\bolddelta^*,\boldtheta^*) \right\| \overset{\mathbb{P}}{\to} \boldzero$. This completes the proof.
	
\end{proof}
\subsection{Proof of Theorem \ref{thm.normality2}}
\label{sec.norm.proof}
\begin{proof}
	 Due to Assumption \ref{assum.unifor.conv}, it can be seen that $\bar{\boldH} \overset{\mathbb{P}}{\to} \mathbb{E}_q\left[\bar{\boldH}\right]$. Moreover, as $\bar{\boldtheta} \overset{\mathbb{P}}{\to} \boldtheta^*$ and $\bar{\bolddelta} \overset{\mathbb{P}}{\to} \boldzero$ (proved in Theorem \ref{thm.main}), we can see $\mathbb{E}_q\left[\bar{\boldH}\right] \overset{\mathbb{P}}{\to} \mathbb{E}_q\left[\boldH^*\right]$ due to continuous mapping. 
	 Thus $\bar{\boldH} = \mathbb{E}_q\left[\boldH^*\right] + o_p(1)$.
	 From now on, for simplicity, let us denote $-\mathbb{E}_q\left[\boldH^*\right]$ as $\boldI$ \footnote{$\boldI$ for ``information matrix''. Do not confuse with the identify matrix which is denoted as $\mathrm{Iden}$ in this paper}. 
	 
     We again write the optimality condition of \eqref{eq.thm.obj} and apply asymptotic expansion at $(\bolddelta^*\equiv\boldzero, \boldtheta^*)$:
	 \begin{align}
	 \nabla_\bolddelta \ell(\hat{\bolddelta},\hat{\boldtheta}) &= \boldzero = \nabla_\bolddelta \ell(\bolddelta^*,\boldtheta^*) + (-\boldI_{11} + o_p(1)) (\hat{\bolddelta} - \bolddelta^*)
	 + (-\boldI_{12} + o_p(1)) (\hat{\boldtheta} - \boldtheta^*)\label{eq.opt.5}\\
	 \nabla_\boldtheta \ell(\hat{\bolddelta},\hat{\boldtheta}) &= \boldzero = \nabla_\boldtheta \ell(\bolddelta^*,\boldtheta^*) + (-\boldI_{21} + o_p(1)) (\hat{\bolddelta} - \bolddelta^*)
	 + (-\boldI_{22} + o_p(1)) (\hat{\boldtheta} - \boldtheta^*)\label{eq.opt.6}.
	 \end{align}
	Note we have replaced all $\bar{\boldH}$ with $-\boldI + o_p(1)$, and $o_p(1)$ will be ignored in future algebraic calculations.
	
	We now get an asymptotic version of \eqref{eq.thetadiff.rep}:
	\begin{align*}
	\sqrt{n_q}\left(\hat{\boldtheta} - \boldtheta^*\right) &\rightsquigarrow -\left(\boldI_{21} \boldI_{11}^{-1}\boldI_{12} - \boldI_{22}\right)^{-1}\boldI_{21} \boldI_{11}^{-1} \nabla_\bolddelta \ell(\bolddelta^*,\boldtheta^*)\cdot \sqrt{n_q} \\
	&= -\left(\boldI_{21} \boldI_{11}^{-1}\boldI_{12} \right)^{-1}\boldI_{21} \boldI_{11}^{-1} \nabla_\bolddelta \ell(\bolddelta^*,\boldtheta^*)\cdot \sqrt{n_q}
	\end{align*}
	The last equality is due to $\boldI_{22} \equiv \boldzero$.

	Noticing that $\boldI_{11}^{-1}\nabla_\bolddelta \ell(\bolddelta^*,\boldtheta^*)\cdot \sqrt{n_q}$ is a sum of independent random variables with zero mean and covariance $-\boldI_{11}^{-1}$.  Applying CLT on $\boldI_{11}^{-1} \nabla_\bolddelta \ell(\bolddelta^*,\boldtheta^*)\cdot \sqrt{n_q}$ yields
	\begin{align*}
	\boldI_{11}^{-1} \nabla_\bolddelta \ell(\bolddelta^*,\boldtheta^*) \rightsquigarrow 
	\mathcal{N}\left(\boldzero, -\boldI^{-1}_{11}\right),
	\end{align*}
	thus
	\begin{align*}
		\sqrt{n_q}\left(\hat{\boldtheta} - \boldtheta^*\right) \rightsquigarrow \mathcal{N}\left[\boldzero, \left(-\boldI_{21}\boldI_{11}^{-1}\boldI_{12}\right)^{-1}\right].
	\end{align*}
\end{proof}

\subsection{Proof of Lemma \ref{cor.mono}}
\label{cor.mono.proof}
\begin{proof}
	Let us shorten the Stein feature vector					 $T_{\boldtheta} \boldf(\boldx)$ as $\boldt(\boldx;\boldtheta) \in \mathbb{R}^{b}$ and $\boldt$ as $\boldt(\boldx;\boldtheta^*)$.
	We start by computing each factors in the variance.
	Since $r_\boldtheta(\boldx;\bolddelta^*)=1$ holds for all $\boldtheta$, we have
	$\nabla_\boldtheta r_\boldtheta(\boldx;\bolddelta^*)=\boldzero$. 
	Then, we have 
	\begin{align*}
	-\mathbb{E}_q\left[\boldH^*_{\bolddelta,\bolddelta}\right]
	&=
	-\mathbb{E}_q\left[\nabla^2_{\bolddelta}\log{r_{\boldtheta^*}(\boldx;\bolddelta^*)}\right]\\
	&=
	\mathbb{E}_q\left[\frac{1}{r(\boldx;\bolddelta^*,\boldtheta^*)^2}\boldt \boldt^{\top}\right]=  \mathbb{E}_q[\boldt \boldt^{\top}]\in\mathbb{R}^{b\times b},\\
	\mathbb{E}_q[\boldH^*_{\boldtheta,\bolddelta}]
	&=
	\mathbb{E}_q\left[\frac{1}{r}{\nabla_\boldtheta}\boldt(\boldx;\boldtheta^*)^{\top}
	-\frac{1}{r^2}\nabla_\boldtheta r_{\boldtheta^*}(\boldx;\bolddelta^*) \boldt(\boldx;\boldtheta^*)^{\top}\right]\\
	&=
	\mathbb{E}_q\left[{\nabla_\boldtheta}\boldt(\boldx;\boldtheta^*)^{\top}\right]\in\mathbb{R}^{\mathrm{dim}(\boldtheta) \times b}. 
	\end{align*}
	Since the equality $\mathbb{E}_{p_\boldtheta}[\boldt(\boldx;\boldtheta)]= \boldzero$ holds for all $\boldtheta$, we have
	${\nabla_ \boldtheta}\mathbb{E}_{p_\boldtheta}[\boldt(\boldx;\boldtheta)]=\boldzero$. Exchangeability of the integration and the derivative yields 
	\begin{align*}
    {\nabla_ \boldtheta}\mathbb{E}_{p_\boldtheta}[\boldt(\boldx;\boldtheta)]
	=\mathbb{E}_{p_\boldtheta}\left[\bolds(\boldx; \boldtheta) \boldt(\boldx;\boldtheta)^{\top}\right]+
	\mathbb{E}_{p_\boldtheta}\left[\nabla_ \boldtheta\boldt(\boldx;\boldtheta)^{\top}\right]=\boldzero. 
	\end{align*}
	As a result, we obtain
	\[
	\mathbb{E}_q[\boldH^*_{\boldtheta,\bolddelta}]= 
	-\mathbb{E}_{q}[\bolds\boldt^{\top}].
	\] 
\end{proof}

\subsection{Proof of Theorem \ref{thm.aic.2}}
\label{sec.aic2.proof}
\begin{proof}
	Use Taylor series to expand $\mathbb{E}_q\left[\ell(\hat{\bolddelta}, \hat{\boldtheta})\right]$ on $(\boldtheta^*, \bolddelta^*)$, we get
	\begin{align}
	\mathbb{E}_q \left[\ell(\hat{\bolddelta},\hat{\boldtheta})\right] = & \mathbb{E}_q \left[\ell(\bolddelta^*,\boldtheta^*)\right] + \nabla_\bolddelta \mathbb{E}_q \left[\ell(\bolddelta^*,\boldtheta^*)\right]^\top\left[\hat{\bolddelta}-\bolddelta^*\right] 
	+\nabla_\boldtheta \mathbb{E}_q \left[\ell(\bolddelta^*,\boldtheta^*)\right]^\top\left[\hat{\boldtheta}-\boldtheta^*\right] \notag\\
	+& \frac{1}{2} \left[\hat{\boldeta}-\boldeta^*\right]^\top \nabla^2_\boldeta \mathbb{E}_q \left[\ell(\bar{\boldeta})\right]\left[\hat{\boldeta}-\boldeta^*\right]\notag\\
	=& 0 + 0  + 0 + \frac{1}{2} \left[\hat{\boldeta}-\boldeta^*\right]^\top \nabla^2_\boldeta \mathbb{E}_q \left[\ell(\bar{\boldeta})\right]\left[\hat{\boldeta}-\boldeta^*\right] \label{eq.aic.11}
	\end{align}
	where we denote $\boldeta := \begin{bmatrix}
	\bolddelta\\
	\boldtheta
	\end{bmatrix}$ for short and $\bar{\boldeta}$ is defined in between $\hat{\boldeta}$ and $\boldeta^*$ in an element-wise fashion. The second equality is due to $\bolddelta^* = \boldzero$ and $ \mathbb{E}_q \left[\nabla_\bolddelta\ell(\bolddelta^*,\boldtheta^*)\right] = \boldzero$, which is given by Stein identity.
	Similarly we can expand 
	\begin{align}
	\ell(\hat{\bolddelta},\hat{\boldtheta}) = &  \nabla_\bolddelta  \ell(\bolddelta^*,\boldtheta^*)^\top\left[\hat{\bolddelta}-\bolddelta^*\right] 
	+ \frac{1}{2} \left[\hat{\boldeta}-\boldeta^*\right]^\top \nabla^2_\boldeta \ell(\bar{\bar{\boldeta}})\left[\hat{\boldeta}-\boldeta^*\right], \label{eq.aic.22}
	\end{align}
	where $\bar{\bar{\boldeta}}$ is similarly defined as $\bar{\boldeta}$.
	It can be seen that $\nabla^2_\boldeta \ell(\bar{\bar{\boldeta}}) \overset{\mathbb{P}}{\to} 
	-\boldI$ and $\nabla^2_\boldeta\mathbb{E}_q \left[\ell\bar{\boldeta}\right] \overset{\mathbb{P}}{\to} 
	-\boldI$ due to Assumption \ref{assum.unifor.conv} and our consistency results. Taking the difference between \eqref{eq.aic.11} and \eqref{eq.aic.22} after multiplying $n_q$ yields
	\begin{align*}
	n_q\mathbb{E}_q \left[\ell(\hat{\bolddelta},\hat{\boldtheta})\right] - n_q\ell(\hat{\bolddelta},\hat{\boldtheta}) = - n_q\nabla_\bolddelta  \ell(\bolddelta^*,\boldtheta^*)^\top\left[\hat{\bolddelta}-\bolddelta^*\right] + o_p(1).
	\end{align*}
	Substitute $(\hat{\bolddelta}-\bolddelta^*)$ with \eqref{eq.opt.3} we get 
	\begin{align*}
	n_q\mathbb{E}_q \left[\ell(\hat{\bolddelta},\hat{\boldtheta})\right] - n_q\ell(\hat{\bolddelta},\hat{\boldtheta}) = n_q\nabla_\bolddelta  \ell(\bolddelta^*,\boldtheta^*)^\top\left[\bar{\boldH}_{11}^{-1}\nabla_\bolddelta \ell(\bolddelta^*,\boldtheta^*)
	+ \bar{\boldH}_{11}^{-1}\bar{\boldH}_{12} (\hat{\boldtheta} - \boldtheta^*)\right] + o_p(1).
	\end{align*}
	Substitute $(\hat{\boldtheta} - \boldtheta^*)$ using \eqref{eq.thetadiff.rep}, we get 
	\begin{align}
	n_q\mathbb{E}_q \left[\ell(\hat{\bolddelta},\hat{\boldtheta})\right] - n_q\ell(\hat{\bolddelta},\hat{\boldtheta}) = & n_q\nabla_\bolddelta  \ell(\bolddelta^*,\boldtheta^*)^\top\bar{\boldH}_{11}^{-1}\nabla_\bolddelta \ell(\bolddelta^*,\boldtheta^*)\notag\\
	-&n_q\nabla_\bolddelta \ell(\bolddelta^*,\boldtheta^*)^\top\bar{\boldH}_{11}^{-1}\bar{\boldH}_{12}	\left[\bar{\boldH}/\bar{\boldH}_{22}\right]^{-1} \bar{\boldH}_{21}\bar{\boldH}_{11}^{-1} \nabla_\bolddelta \ell(\bolddelta^*,\boldtheta^*)+ o_p(1)\label{eq.normality.1}
	\end{align}
	Replacing submatrices of $\bar{\boldH}_{a,b}$ using submatrices of $-\boldI_{a,b}$ in \eqref{eq.normality.1} and using the fact that $\boldI_{22}\equiv \boldzero$ (due to $\bolddelta^* = \boldzero$),
	\begin{align}
	n_q\mathbb{E}_q \left[\ell(\hat{\bolddelta},\hat{\boldtheta})\right] & - n_q\ell(\hat{\bolddelta},\hat{\boldtheta}) =  -\sqrt{n_q}\nabla_\bolddelta  \ell(\bolddelta^*,\boldtheta^*)^\top\boldI_{11}^{-1}\nabla_\bolddelta \ell(\bolddelta^*,\boldtheta^*)\sqrt{n_q} \notag\\
	& +\sqrt{n_q}\nabla_\bolddelta \ell(\bolddelta^*,\boldtheta^*)^\top\boldI_{11}^{-1}\boldI_{12}	\left[\boldI_{21}\boldI_{11}^{-1}\boldI_{12} \right]^{-1} \boldI_{21}\boldI_{11}^{-1} \nabla_\bolddelta \ell(\bolddelta^*,\boldtheta^*)\sqrt{n_q} + o_p(1)\label{eq.delldelta}
	\end{align}
	Taking the expectation, 
	\begin{align*}
	n_q\mathbbE\left\{\mathbb{E}_q \left[\ell(\hat{\bolddelta},\hat{\boldtheta})\right] - \ell(\hat{\bolddelta},\hat{\boldtheta}) \right\} =& -\mathrm{trace}(\boldI_{11}\boldI_{11}^{-1}) + \mathrm{trace}({\boldI_{11}^{-1}\boldI_{12}	\left[\boldI_{21}\boldI_{11}^{-1}\boldI_{12} \right]^{-1} \boldI_{21}})+ o_p(1) \\
	=& -\mathrm{rank}(\boldI_{11}) + \mathrm{rank}\left(\boldI_{21}\boldI_{11}^{-1}\boldI_{12}\right)+ o_p(1).
	\end{align*}
	In the case when $\boldI_{11} \in \mathbb{R}^{b \times b}, \boldI_{12}\in \mathbb{R}^{b \times \mathrm{dim}(\boldtheta)}$ are full-rank and $\mathrm{dim}(\boldtheta) \le b$, $\mathrm{rank}(\boldI_{11}) = b$ and $\mathrm{rank}\left(\boldI_{21}\boldI_{11}^{-1}\boldI_{12}\right) = \mathrm{dim}(\boldtheta)$, which completes the proof.
\end{proof}

\subsection{The Asymptotic Distribution of $2n_q \ell(\hat{\bolddelta},\hat{\boldtheta})$}
\label{sec.gof.proof}
We show $2n_q\ell(\hat{\bolddelta},\hat{\boldtheta})$ follows a $\chi^2$ distribution based on previously assumed assumptions.
\begin{thm}
\label{thm.asym.gof}
	Suppose Assumption \ref{eq.model}, \ref{assum.info.hessian},  \ref{ass.concentration} and \ref{assum.unifor.conv} holds,  $\mathbb{E}_q\left[\boldH^*_{\bolddelta,\bolddelta}\right]$ is invertible and $ \mathbb{E}_q\left[\boldH^*_{\bolddelta,\boldtheta}\right]$ are full-rank and $\mathrm{dim}(\boldtheta)< b$, then  $2n_q\ell(\hat{\bolddelta},\hat{\boldtheta}) \rightsquigarrow \chi^2(b - \mathrm{dim}(\boldtheta)).$
\end{thm}

\begin{proof}
	First we expand $2n_q\ell(\hat{\bolddelta},\hat{\boldtheta})$ using mean value theorem:
	\begin{align}
	2n_q\ell(\hat{\bolddelta},\hat{\boldtheta}) &=  2n_q \nabla_\bolddelta\ell(\bolddelta^*,\boldtheta^*)^\top d\bolddelta + n_q d\bolddelta \bar{\boldH}_{11} d\bolddelta + n_q d\bolddelta \bar{\boldH}_{12} d\boldtheta + n_q d\boldtheta \bar{\boldH}_{21} d\bolddelta + n_q d\boldtheta \bar{\boldH}_{22} d\boldtheta \label{eq.gof.1}
	\end{align}
	where $d\boldt$ is short for $\hat{\boldt} - \boldt^*$. Note $\ell(\bolddelta^*, \boldtheta^*) = 0$. Now we analyze each term.
	
	From the proof in Section \ref{sec.aic2.proof} we know
	\begin{align}
	2n_q \nabla_\bolddelta\ell(\bolddelta^*,\boldtheta^*)^\top d\bolddelta = & 2n_q\nabla_\bolddelta  \ell(\bolddelta^*,\boldtheta^*)^\top\boldI_{11}^{-1}\nabla_\bolddelta \ell(\bolddelta^*,\boldtheta^*) \notag\\
	-& 2n_q\nabla_\bolddelta \ell(\bolddelta^*,\boldtheta^*)^\top\boldI_{11}^{-1}\boldI_{12}	\left[\boldI_{21}\boldI_{11}^{-1}\boldI_{12} \right]^{-1} \boldI_{21}\boldI_{11}^{-1} \nabla_\bolddelta \ell(\bolddelta^*,\boldtheta^*) + o_p(1)\label{eq.gof.2}.
	\end{align}
	With the help of \eqref{eq.opt.3} and \eqref{eq.thetadiff.rep} and a few algebra we can see that 
	\begin{align}
	n_q d\bolddelta \boldI_{11} d\bolddelta &= n_q\nabla_\bolddelta  \ell(\bolddelta^*,\boldtheta^*)^\top\boldI_{11}^{-1}\nabla_\bolddelta \ell(\bolddelta^*,\boldtheta^*) \notag\\
	&-n_q\nabla_\bolddelta \ell(\bolddelta^*,\boldtheta^*)^\top\boldI_{11}^{-1}\boldI_{12}	\left[\boldI_{21}\boldI_{11}^{-1}\boldI_{12} \right]^{-1} \boldI_{21}\boldI_{11}^{-1} \nabla_\bolddelta \ell(\bolddelta^*,\boldtheta^*)+ o_p(1).\label{eq.gof.3}
	\end{align}
	Similar calculations also show $n_q d\bolddelta \bar{\boldH}_{12} d\boldtheta = n_q d\boldtheta \bar{\boldH}_{21} d\bolddelta = o_p(1)$ and $n_q d\boldtheta \bar{\boldH}_{22} d\boldtheta = o_p(1)$. Combine \eqref{eq.gof.1}, \eqref{eq.gof.2} and \eqref{eq.gof.3}, we can see that 
	\begin{align*}
	2n_q\ell(\hat{\bolddelta},\hat{\boldtheta}) =&  n_q\nabla_\bolddelta  \ell(\bolddelta^*,\boldtheta^*)^\top\boldI_{11}^{-1}\nabla_\bolddelta \ell(\bolddelta^*,\boldtheta^*)\\
	-& n_q\nabla_\bolddelta \ell(\bolddelta^*,\boldtheta^*)^\top\boldI_{11}^{-1}\boldI_{12}	\left[\boldI_{21}\boldI_{11}^{-1}\boldI_{12} \right]^{-1} \boldI_{21}\boldI_{11}^{-1} \nabla_\bolddelta \ell(\bolddelta^*,\boldtheta^*) + o_p(1)\\
	=& \sqrt{n_q}\nabla_\bolddelta \ell(\bolddelta^*,\boldtheta^*)^\top 
	\boldI_{11}^{-1}\left\{
	\mathrm{Iden} - \boldI_{12}	\left[\boldI_{21}\boldI_{11}^{-1}\boldI_{12} \right]^{-1} \boldI_{21}\boldI_{11}^{-1}
	\right\}
	\nabla_\bolddelta \ell(\bolddelta^*,\boldtheta^*)\sqrt{n_q}+ o_p(1),
	\end{align*}
	where $\mathrm{Iden}$ is identify matrix. 
	Denote $\mathrm{Iden} - \boldI_{12}	\left[\boldI_{21}\boldI_{11}^{-1}\boldI_{12} \right]^{-1} \boldI_{21}\boldI_{11}^{-1}$ as $A$. 
	One can verify that 
	$\nabla_\bolddelta \ell(\bolddelta^*,\boldtheta^*)^\top \boldI_{11}^{-1}A$ has covariance $\boldI_{11}^{-1}A$ \footnote{Some calculations show $A^\top \boldI_{11}^{-1}A = \boldI_{11}^{-1}A$.}. By checking the eigenvalues of $A$\footnote{$\mathrm{eig}(\mathrm{Iden} - T) = 1 - \mathrm{eig}(T)$ and $\mathrm{eig}(ST) = \mathrm{eig}(TS)$.}, it can be seen that 
	$
	\mathrm{rank}(A) = db - \mathrm{dim}(\boldtheta)
	$
	and assuming $\boldI_{11}^{-1}$ is full rank, $\mathrm{rank}(\boldI_{11}^{-1}A) = db - \mathrm{dim}(\boldtheta)$. 
	Therefore $\sqrt{n_q}\nabla_\bolddelta \ell(\bolddelta^*,\boldtheta^*)^\top \boldI_{11}^{-1}A$ is asymptotically a degenerated multivariate normal variable with covariance matrix $\boldI_{11}^{-1}A$. 
	
	We can rewrite $\sqrt{n_q}\nabla_\bolddelta \ell(\bolddelta^*,\boldtheta^*)^\top 
	\boldI_{11}^{-1}A
	\nabla_\bolddelta \ell(\bolddelta^*,\boldtheta^*)\sqrt{n_q}$ as \[\sqrt{n_q}\nabla_\bolddelta \ell(\bolddelta^*,\boldtheta^*)^\top 
	\boldI_{11}^{-1}A\left[\boldI_{11}^{-1}A\right]^{+}\boldI_{11}^{-1}A
	\nabla_\bolddelta \ell(\bolddelta^*,\boldtheta^*)\sqrt{n_q},\] where $T^{+}$ is the pseudoinverse. This quadratic form has a  $\chi^2$ distribution with degree of freedom $\mathrm{rank}(\boldI_{11}^{-1}A) = db - \mathrm{dim}(\boldtheta)$.
\end{proof}

\subsection{Proof of \Cref{prop.lag.dual}}
\label{sec.proof.dual}
\begin{proof}
	We convert the SDRE problem $\eqref{eq.steinDR}$ as the following equivalent problem: 
	\begin{align*}
	\max_{\bolddelta,\boldepsilon}   \sum_{i=1}^{n_q} \log \epsilon_i 
	\mathrm{ ~~s.t. : } \forall i\in \{1\dots n_q\}, ~\bolddelta^\top  T_{\boldtheta}\boldf(\boldx_q^{(i)}) + 1= \epsilon_i.
	\end{align*}
	Let us introduce Lagrangian multipliers $\mu_1 \dots \mu_{n_q}$ over all the constraints. 
	We can write the Lagrangian: 
	\begin{align}
	\label{eq.lag.1}
	\min_{\boldmu}\max_{\bolddelta,\boldepsilon} \sum_{i=1}^{n_q} \left(\log \epsilon_i \right) - \sum_{i=1}^{n_q} \mu_i \left(\bolddelta^\top  T_{\boldtheta}\boldf(\boldx_q^{(i)}) + 1 - \epsilon_i \right)
	\end{align}
	Solve the inner max problem with respect to $\boldepsilon$, 
	\begin{align}
	\label{eq.lag.2}
	\max_\boldepsilon \sum_{i=1}^{n_q} \log \epsilon_i + \mu_i \epsilon_i = \sum_{i=1}^{n_q} [- (\log -\mu_i) -1], 
	\end{align}
	when $\epsilon_i = -\frac{1}{\mu_i}$. 
	This also implies the relationship between the dual parameter $\mu_i$ and the primal parameter $\bolddelta$: $r_\boldtheta (\boldx_q^{(i)};\bolddelta) = \bolddelta^\top T_{p_{\boldtheta}}\boldf(\boldx_q^{(i)})+1 = \epsilon_i = -\frac{1}{\mu_i}$.
	
	The inner optimization with respect to $\bolddelta$, i.e., $\max_\bolddelta - \sum_{i=1}^{n_q} \mu_i \bolddelta^\top  T_{\boldtheta}\boldf(\boldx_q^{(i)})$ is a linear programming and is only bounded when $\sum_{i=1}^{n_q} \mu_i  T_{\boldtheta}\boldf(\boldx_q^{(i)})=\boldzero$ and achieves the optimal value 0. 
	
	Substituting the optimal values of these two maximization results into the Lagrangian and adding  constraint $\sum_{i=1}^{n_q} \mu_i  T_{\boldtheta}\boldf(\boldx_q^{(i)})=\boldzero$ gives the Lagrangian dual \eqref{eq.lag.0}. Moreover, the primal problem in \eqref{eq.steinDR} is concave, we can verify the Slater's condition holds at $\bolddelta = \boldzero, \boldepsilon = \boldone$ thus the strong duality holds. 
\end{proof}

\end{document}